\documentclass[sigconf]{acmart}
\AtBeginDocument{%
  }

\acmConference[KDD '26]{Proceedings of the 32nd ACM SIGKDD Conference on Knowledge Discovery and Data Mining V.2}{August 09--13, 2026}{Jeju Island, Republic of Korea}
\acmBooktitle{Proceedings of the 32nd ACM SIGKDD Conference on Knowledge Discovery and Data Mining V.2 (KDD '26), August 09--13, 2026, Jeju Island, Republic of Korea}
\usepackage{multirow}
\usepackage{booktabs}
\usepackage{bm}
\usepackage{adjustbox}
\usepackage{tabularx}
\usepackage{booktabs}
\usepackage{algorithm}
\usepackage{algpseudocode}
\usepackage{balance}

\usepackage[table]{xcolor}
\definecolor{color01}{RGB}{211, 204, 241}
\definecolor{color02}{RGB}{227, 242, 253}

\usepackage{diagbox}

\usepackage{tikz}

\newcommand{\databar}[1]{%
  \begin{tikzpicture}[baseline=-0.5ex]
    \fill[blue!30] (0,0) rectangle (#1/100*0.5,0.2);
    \draw[black,thin] (0,0) rectangle (0.5,0.2);
    \node[anchor=west,inner sep=1pt] at (0.55,0.1) {\small #1};
  \end{tikzpicture}%
}

\begin{document}


\title{AbstainGNN: Teaching Graph Neural Networks to Abstain for Graph Classification}

\author{Xixun Lin}
\orcid{0009-0004-6645-0597}
\affiliation{%
  \institution{Institute of Information Engineering, Chinese Academy of Sciences}
  \city{Beijing}
  \country{China}
}
\email{agblinxx@gmail.com}

\author{Zhiheng Zhou}
\orcid{0000-0002-6743-7835}
\affiliation{%
  \institution{School of Mathematics and Statistics, Shandong University}
  \city{Weihai}
  \country{China}
}
\email{zhouzhiheng@amss.ac.cn}
\authornote{ denotes corresponding author.}

\author{Zhengyin Zhang}
\orcid{0009-0001-1640-7150}
\affiliation{%
  \institution{China University of Mining and Technology-Beijing}
  \city{Beijing}
  \country{China}
}
\email{2210730330@student.cumtb.edu.cn}

\author{Yancheng Chen}
\orcid{0009-0003-7935-9904}
\affiliation{%
  \institution{Academy of Mathematics and Systems Science, Chinese Academy of Sciences}
  \city{Beijing}
  \country{China}
}
\email{chenyancheng22@mails.ucas.ac.cn}

\author{Shuai Zhang}
\orcid{0009-0001-3152-3651}
\affiliation{%
  \institution{Academy of Mathematics and Systems Science, Chinese Academy of Sciences}
  \city{Beijing}
  \country{China}
}
\email{zhangshuai2021@amss.ac.cn}

\author{Ge Zhang}
\orcid{0000-0001-6009-780X}
\affiliation{%
  \institution{School of Information and Intelligent Science, Donghua University}
  \city{Shanghai}
  \country{China}
}
\email{gezhang@dhu.edu.cn}

\author{Shichao Zhu}
\orcid{0000-0002-1755-4267}
\affiliation{%
  \institution{TikTok, ByteDance Inc.}
  \city{San Jose}
  \country{United States}
}
\email{shichao.szhu@gmail.com}

\author{Lixin Zou}
\orcid{0000-0001-6755-871X}
\affiliation{%
  \institution{School of Cyber Science and Engineering, Wuhan University}
  \city{Wuhan}
  \country{China}}
\email{zoulixin@whu.edu.cn}

\author{Chuan Zhou}
\orcid{0000-0001-9958-8673}
\affiliation{%
  \institution{Academy of Mathematics and Systems Science, Chinese Academy of Sciences}
  \city{Beijing}
  \country{China}
}
\email{zhouchuan@amss.ac.cn}

\author{Peng Zhang}
\orcid{0000-0001-7973-2746}
\affiliation{%
  \institution{Cyberspace Institute of Advanced Technology, Guangzhou University}
  \city{Guangzhou}
  \country{China}
}
\email{p.zhang@gzhu.edu.cn}

\author{Shirui Pan}
\orcid{0000-0003-0794-527X}
\affiliation{%
  \institution{Griffith University}
  \city{Brisbane}
  \country{Australia}
}
\email{s.pan@griffith.edu.au}

\author{Yanan Cao}
\orcid{0000-0003-3534-1094}
\affiliation{%
  \institution{Institute of Information Engineering, Chinese Academy of Sciences}
  \city{Beijing}
  \country{China}
}
\email{caoyanan@iie.ac.cn}
\authornotemark[1]

\renewcommand{\shortauthors}{Xixun Lin et al.}

\begin{abstract}

Graph classification is a core task in graph data mining with widespread real-world applications. Recent advances in graph neural networks (GNNs) have led to substantial performance improvements for graph classification. However, existing GNNs are typically forced to make predictions even under high uncertainty or unknown conditions, resulting in unreliable decisions that can severely impact downstream tasks, particularly in safety-critical scenarios. To address this critical limitation, we propose AbstainGNN, a novel and theory-driven framework for graph classification with abstention, which enables GNNs to reject uncertain predictions instead of producing incorrect decisions. Specifically, AbstainGNN explicitly models both the predictive function and the abstention function, allowing for effective utilization of graph structural information. Moreover, unlike existing heuristic abstention methods, we theoretically characterize the trade-off between classification errors and rejection costs from a PAC-Bayesian generalization perspective, and derive a unified learning objective for model optimization. Guided by this theoretical insight, we further develop an efficient two-stage training strategy consisting of predictive function warm-start and abstention function calibration. Extensive experiments on five benchmark datasets show that AbstainGNN outperforms existing abstention methods, achieving superior classification performance under the same rejection rates.

\end{abstract}

\ccsdesc[500]{Computing methodologies~Machine learning}
\ccsdesc[500]{Information systems~Data mining}


\keywords{Graph Classification, Learning with Abstention}


\maketitle

\newcommand\kddavailabilityurl{https://doi.org/10.5281/zenodo.20422037}
\ifdefempty{\kddavailabilityurl}{}{
\begingroup\small\noindent\raggedright\textbf{Resource Availability:}\\
The source code of this paper has been made publicly available at \url{\kddavailabilityurl}.
\endgroup
}

\section{Introduction}
Graph classification is a fundamental task in graph data mining, which aims to assign a semantic label to an entire graph by jointly modeling graph topologies, node attributes, and edge relations~\cite{xia2021graph}. Graph classification plays an important role in a wide range of real-world applications, including molecular property prediction~\cite{wu2020comprehensive,lin2022structure}, material  structure analysis~\cite{xie2018crystal,park2020developing}, and social network analysis~\cite{ying2018hierarchical,guo2025fgdgnn}.  
Recent advances in graph neural networks (GNNs) have substantially boosted graph classification performance. These improvements mainly stem from message passing and graph pooling mechanisms to obtain expressive graph-level representations~\cite{mesquita2020rethinking,ying2024boosting,buterez2025end}. 
\par
Despite notable progress in graph classification, most existing GNNs still struggle to make reliable decisions. Graph samples are often forced into specific classes even when predictions are made with low confidence, ambiguous structural evidence, or weak class separation~\cite{wu2022trustworthy,yang2024balanced,xie2024exploring}. Such unreliable predictions can propagate severe errors to downstream tasks, compromising the trustworthiness of GNN-based decision systems~\cite{zhang2024trustworthy,lin2025conformal,lin2025generative}. For example, in molecular property prediction, GNNs may be compelled to classify a compound as toxic or non-toxic even when the underlying evidence is highly uncertain. Such decisions can mislead early-stage drug screening, either discarding valuable compounds or advancing risky ones into expensive evaluations. Consequently, the lack of reliable decision mechanisms significantly restricts the deployment of GNNs in high-stakes and safety-critical applications.
\par 
In fact, \textbf{Learning with Abstention}, also known as classification with rejection, provides an ideal paradigm to address the above significant limitation, enabling the model to abstain from making predictions in uncertain cases~\cite{chow1970,grandvalet2008support,cortes2016learning,ni2019calibration,charoenphakdee2021classification}. The core mechanism is to identify samples that are likely to be misclassified and defer them to human experts or downstream procedures for later decision-making. This paradigm is particularly appealing in safety-critical scenarios, where the cost of an incorrect decision can far outweigh the cost of abstention. Nevertheless, existing methods face the following two major issues in the context of graph classification: 
\begin{itemize}
    \item The majority of abstention methods have been developed for computer vision, leaving graph-structured data largely underexplored and inadequately modeled. Consequently, they fail to capture the rich topological and relational information inherent in graphs.
    \item Most existing methods rely heavily on heuristic model designs, offering limited theoretical insights into abstention mechanisms specifically for graphs. This lack of theoretical grounding in graph domains hinders the development of more principled and effective abstention mechanisms for graph classification. 
\end{itemize}
To address these issues, we propose \textbf{AbstainGNN}, a novel abstention framework for graph classification. Specifically, AbstainGNN explicitly designs both the predictive function and the abstention function for graph-structured data, enabling more effective modeling of graph structural information. Moreover, AbstainGNN is a theory-driven learning framework rather than a heuristic model design. From the perspective of generalization bounds~\cite{neyshabur2017pac}, we theoretically characterize how AbstainGNN achieves a principled trade-off between classification errors and rejection costs. This analysis leads to a unified learning objective and an efficient two-stage learning stages, consisting of \emph{Predictive Function Warm-start} and \emph{Abstention Function Calibration}. In the first stage, the warm-start strategy drives the predictive function into a small-gradient region, facilitating the more stable model optimization. In the second stage, the learned predictive function is used to calibrate the abstention function, ensuring that confidence scores are well aligned with model predictions. Extensive experiments and analyses on multiple benchmark datasets demonstrate the significant advantages of our model over existing abstention methods. In general, the main contributions of our work are summarized as follows, 
\begin{itemize}
    \item {\bf Problem Statement.} We are the first to investigate the problem of graph classification with abstention, which equips GNNs with the ability to abstain from making predictions on uncertain or unknown graph samples, significantly improving their reliability and applicability in safety-critical domains.
    
    \item {\bf Theoretical Analysis.} We derive PAC-Bayesian generalization bounds for AbstainGNN, revealing that minimizing the intra-class variance of graph-level representations is crucial for graph classification with abstention. Building on this insight, we propose a unified learning objective for model optimization.
    
    \item {\bf Model Implementation.} Guided by our theoretical analysis, we develop an efficient model implementation of AbstainGNN. In particular, we introduce a simple yet effective global class-cluster adjustment strategy, which alleviates the estimation bias caused by batch-wise updates.

    \item {\bf Experimental Verification.} We conduct comprehensive experiments on five real-world datasets, demonstrating that our method achieves significantly superior classification performance under the same rejection rates. The most notable improvement is that our model yields an average relative risk reduction of 16.8\% on MUTAG\footnote{Our source code is available at 
    https://github.com/ZZY565/AbstainGNN.}.  
\end{itemize}

\section{Related works}

\subsection{Graph Classification}
Graph classification has been extensively studied in the data mining community. Early methods primarily rely on graph kernels, which compute the similarity between graphs based on substructures such as walks, paths, or subtrees~\cite{kashima2003marginalized,borgwardt2005propagated,kriege2020graphkernels}. While these methods are theoretically well-founded, they often require careful feature engineering and struggle to scale to large graphs. With the emergence of GNNs, graph classification has increasingly shifted toward end-to-end learning manners~\cite{xu2018powerful,errica2019fair,li2024graph}. In most GNN-based methods, node representations are iteratively updated by aggregating information from their neighbors, and are subsequently pooled to produce graph-level representations. Such approaches have achieved state-of-the-art performance across a variety of applications, including drug discovery~\cite{gilmer2017neural,wu2018moleculenet} and protein prediction~\cite{dobson2003protein,shervashidze2011weisfeiler}. Despite these advances, existing graph classification methods are typically designed to produce predictions for all input graphs, forcing the model to make decisions even on low-confidence graph samples. This design drawback introduces potential risks when deploying GNNs in high-risk and safety-critical domains. 

\subsection{Learning with Abstention}
Learning with abstention, also known as classification with rejection, allows the classifier to refrain from making predictions on uncertain inputs, thereby reducing the risk of misclassification in critical applications~\cite{chow2003optimum,huang2020self,zhu2022efficient}. Existing approaches can be classified broadly into two categories: cost-based and coverage-based methods. Cost-based approaches aim to prevent misclassification by providing an option not to make a prediction at the expense of the pre-defined rejection cost~\cite{ramaswamy2018consistent,ni2019calibration,charoenphakdee2021classification}. Coverage-based approaches define coverage as the proportion of samples on which the model chooses to make predictions rather than abstain. Given a coverage, the model seeks to select a subset of examples that maximizes predictive performance while rejecting the remaining uncertain samples~\cite{geifman2017selective,gayen2025predict,geifman2019selectivenet}. \citet{franc2019discriminative} theoretically analyze the equivalence between these two categories. For graph domains, some works~\cite{kuchipudinode,gayen2025predict} have followed the classic abstention mechanism to study rejection behaviors in node-level and dynamic graph scenarios. Overall, learning with abstention has primarily focused on computer vision, and how to design algorithmic and theoretical learning frameworks, especially for graph classification, remains a highly important yet largely underexplored problem.

\section{Background}
\label{sec:Background}
In this section, we describe the basic framework of GNNs for graph classification and discuss their inherent limitations under this framework. Let $G=(\mathcal{V},\mathcal{E})$ denote an input graph, where $\mathcal{V}$ and $\mathcal{E}$ represent the sets of nodes and edges, respectively. The node attribute matrix is denoted by $\mathbf{X} \in \mathbb{R}^{|\mathcal{V}| \times c}$, where $|\mathcal{V}|$ is the numbers of nodes and $c$ is the dimensionality of node attributes. The adjacency matrix is given by $\mathbf{A}$, where each element $\mathbf{A}_{ij} \in \{0,1\}$ indicates whether an edge exists between nodes $i$ and $j$. Graph classification aims to learn a predictive function $f$ that maps the input graph to a probability vector $f(\mathbf{A}, \mathbf{X}) \in [0,1]^{|\mathcal{Y}|}$, from which a predicted label is obtained. $\mathcal{Y}$ is the set of class labels. As discussed above, GNNs have become the dominant frameworks for graph classification, which learn graph-level representations in an end-to-end manner~\cite{wu2021comprehensive}. 
\par
Most GNNs can be abstracted as a composition of message-passing layers, a permutation-invariant readout operator, and a final classification layer. Formally, the entire prediction process can be written as this composite map: 
\begin{align}
\label{GNN_composite_map}
    f(\mathbf{A},\mathbf{X})
    =(C \circ R 
    \circ U^{(L)} \circ \cdots \circ U^{(1)})(\mathbf{A}, \mathbf{X}).
\end{align}
Here, each message-passing layer $U^{(l)}$ follows the calculation mechanism: At the $l$-th layer ($l=1,\dots,L$), each node aggregates messages from its neighbors 
and updates its representation as follows, 

\begin{align}
\mathbf{m}_i^{(l)} &= \sum_{j \in \mathcal{N}(i)} 
    M^{(l)}\!\left(\mathbf{h}_i^{(l-1)},\, \mathbf{h}_j^{(l-1)},\, \mathbf{e}_{ij}\right), \\
\mathbf{h}_i^{(l)} &= U^{(l)}\!\left(\mathbf{h}_i^{(l-1)},\, \mathbf{m}_i^{(l)}\right),
\end{align}
where $\mathcal{N}(i)$ denotes the neighbors of the node $i$, $\mathbf{h}_i^{(l-1)}$ and $\mathbf{h}_j^{(l-1)}$ are the node representations of $i$ and $j$ in the preceding layer, and $\mathbf{e}_{ij}$ is the (optional) edge attribute. $M^{(l)}(\cdot)$ is the message function for aggregating representations and $U^{(l)}(\cdot)$ is the update function to generate the final representation. After $L$ layers of propagation, a permutation-invariant readout operator is performed to obtain the graph-level representation:
\begin{align}
\label{graph_representations}
    \mathbf{h}_G = R\!\left( \{ \mathbf{h}_i^{(L)} | i \in  \mathcal{V} \} \right), 
\end{align}
where $R(\cdot)$ can be the mean/sum pooling or the more sophisticated operator. Finally, a classification layer $\mathrm{C}$ maps  $\mathbf{h}$ to the predicted class distribution $\hat{\mathbf{y}}$: 
\begin{align}
    \hat{\mathbf{y}} = C(\mathbf{h}_G)&=\mathrm{softmax}(\mathbf{Wh_G} + \mathbf{b}),
\end{align}
where $\mathbf{W}$ and $\mathbf{b}$ denote the weight matrix and the bias vector, and we denote the predicted class $\hat{y}$ as 
\begin{equation}
    \hat{y}= \max_{y\in\mathcal{Y}} \hat{\mathbf{y}}=\max_{y\in\mathcal{Y}}f_y(\mathbf{A},\mathbf{X}).
\end{equation}


\paragraph{Current Limitation} Despite the strong performance of existing GNNs, they typically lack the ability to abstain from making predictions on unknown or out-of-distribution samples. This limitation significantly restricts their applicability, especially in risk-sensitive domains. For example, in financial risk assessment or medical diagnosis, a misclassification can lead to substantial losses or even endanger human safety, making it essential for the model to recognize when it is uncertain and avoid forced predictions. Therefore, a more desirable learning paradigm is to incorporate abstention capability into the GNN learning process, enabling GNNs to make confident predictions on reliable samples while abstaining from uncertain ones.


\section{Graph Classification with Abstention}
\label{sec:Graph_Classification_with_Abstention}
To overcome the above limitation, we introduce AbstainGNN, a novel GNN paradigm for graph classification with abstention. Let $\mathcal{G}$ denote a space of graphs. Building upon the predictive function $f: \mathcal{G}\to [0,1]^{|\mathcal{Y}|}$ ($\mathcal{Y}$ is the label set), AbstainGNN introduces an additional graph-level abstention function $g:\mathcal{G}\to[0,1]$, which determines whether a prediction should be made. Concretely, given an input graph $G=(\mathbf{A},\mathbf{X})\in\mathcal{G}$, AbstainGNN outputs
\begin{equation}
(f,g)(\mathbf{A},\mathbf{X})\triangleq
\begin{cases}
f(\mathbf{A},\mathbf{X}), & \mbox{if }g(\mathbf{A}, \mathbf{X})\geq \epsilon,\\
\text{abstain}, &\mbox{if } g(\mathbf{A}, \mathbf{X})< \epsilon, 
\end{cases}
\end{equation}
where $\epsilon$ is a predefined threshold. That is,  AbstainGNN abstains on $G$ if and only if the confidence score $g(\mathbf{A}, \mathbf{X})$ falls below the threshold $\epsilon$, i.e., $g(\mathbf{A},\mathbf{X}) < \epsilon$.
\par
Although abstention can improve the model’s accuracy, excessive abstention, where the model refrains from making predictions, renders it meaningless. Therefore, an ideal classifier with reject option should achieve a good trade-off between \emph{coverage} and \emph{risk}~\cite{chow1970reject,grandvalet2008support,cortes2016boosting}. Given a dataset $\mathcal{S}=\{(G_i,y_i)\}_{i=1}^n \subset \mathcal{G}\times \mathcal{Y}$,
the empirical coverage is defined as 
\begin{equation}
\Phi(g\mid \mathcal{S})\triangleq \frac{1}{n}\sum_{i=1}^n \mathbf{1}(g(\mathbf{A}_i, \mathbf{X}_i)\geq h),
\end{equation}
which measures the fraction of graphs on which the model makes predictions. Here, $\mathbf{1}(\cdot)$ is the indicator function. The empirical risk is defined as the average loss conditioned on making predictions:
\begin{equation}
\Psi(f,g\mid \mathcal{S})\triangleq
\frac{\frac{1}{n}\sum_{i=1}^n \mathbf{1}(g(\mathbf{A}_i, \mathbf{X}_i)\geq h) \ell\!\left(f(\mathbf{A}_i, \mathbf{X}_i),y_i\right)\,}
{\Phi(g\mid  \mathcal{S})} ,
\end{equation}
where $\ell(\cdot,\cdot)$ denotes a classification loss function. Intuitively, coverage measures how often the model makes predictions, while risk quantifies the model's error on the accepted samples. Under this formulation, learning with abstention~\cite{geifman2019selectivenet} aims to minimize the risk while ensuring a target coverage level $c_{\text{t}}$:
\begin{equation}
\label{learning obj}
\min_{(f,g)}\; \Psi(f,g\mid \mathcal{S})
\quad \text{s.t.} \quad
\Phi(g \mid  \mathcal{S}) \ge c_{\text{t}}.
\end{equation}

\subsection{Generalization Analysis of AbstainGNN}
For further analysis, we apply the Lagrangian multiplier method to Eq.\eqref{learning obj} and derive the following per-sample reject-cost loss:
\begin{equation}
\label{eq:reject_cost_loss}
L\!\left(f,g; G, y\right)
= \mathbf{1}\!\left(g(\mathbf{A},\mathbf{X})\ge \epsilon \right) \ell\!\left(f(\mathbf{A},\mathbf{X}),y\right)\,
+\lambda\,\mathbf{1}\!\left(g(\mathbf{A},\mathbf{X})< \epsilon\right),
\end{equation}
where $\lambda>0$ is the Lagrange multiplier. In Eq.(\ref{eq:reject_cost_loss}), the first term penalizes classification errors on graphs that are \emph{accepted} (i.e., predicted), whereas the second term assigns a fixed penalty to \emph{rejected} graphs. By adjusting $\lambda$, we control the trade-off between minimizing the risk on accepted samples
and maintaining a sufficiently high coverage to satisfy the constraint. However, the used hard indicators ($\mathbf{1}(g(\mathbf{A},\mathbf{X})\ge \epsilon)$ and $\mathbf{1}(g(\mathbf{A},\mathbf{X}) < \epsilon)$) make Eq.\eqref{eq:reject_cost_loss} non-smooth and almost everywhere non-differentiable, particularly at the boundary $g(\mathbf{A},\mathbf{X})=\epsilon$. 


\par
To better analyze theoretical properties of AbstainGNN, we replace Eq.\eqref{eq:reject_cost_loss} by a differentiable Max–Hinge loss~\cite{cortes2016learning}. Let $\gamma(\mathbf{A},\mathbf{X})=f_y(\mathbf{A},\mathbf{X})-\mbox{max}_{j\neq y}f_j(\mathbf{A},\mathbf{X})$  denote the classification margin of $f$, the margin loss is given as 
\begin{equation}
\begin{split}
L_{MH}^{\rho,\rho'}(f,g;G,y)
= \max\Big\{
\max\!\bigl[\lambda(1-\beta\frac{g(\mathbf{A},\mathbf{X})}{\rho'}),0\bigr], \\
\max\!\bigl[1+\frac{\alpha}{2}
(\frac{g(\mathbf{A}, \mathbf{X})}{\rho'}-\frac{\gamma(\mathbf{A}, \mathbf{X})}{\rho}),0\bigr]
\Big\},
\end{split}
\label{eq:maxhinge}
\end{equation}  
where $\rho,\rho'>0$ are two parameters associated with the minimum margins of $f$ and $g$ respectively, and $\alpha,\beta>0$ are two balanced factors. 
\par 
Let \(f = C \circ R \circ U\) be an arbitrary GNN described in Section~\ref{sec:Background}, where \(U\) denotes all message-passing layers, \(R\) is the pooling layer, and \(C\in\mathbb{R}^{d\times |\mathcal{Y}|}\) is the classification layer ($d$ is the dimensionality of learned graph-level representations). Let $g$ is our proposed graph-level abstention function and \(\mathcal{S}=\{(G_i,y_i)\}_{i=1}^n\) is the training dataset consisting of \(n\) i.i.d.\ samples drawn from an unknown data distribution \(\mathcal{D}\), then we have the following generalization bound for AbstainGNN: 
\begin{theorem}
\label{thm:abstention_gen}
For any \(\rho,\rho',\alpha,\beta,\lambda>0\), \(\delta\in(0,1)\), and \(u=\max\{d,|\mathcal{Y}|\}\), with probability at least \(1-\delta\) over \(\mathcal{S}\), it holds that
\begin{equation}
\begin{aligned}
\mathbb{E}&_{(G,y)\sim \mathcal{D}}\!\left[L(f,g;G,y)\right]
\leq 
\mathbb{E}_{(G,y)\sim \mathcal{S}}\!\left[L_{MH}^{\rho,\rho'}(f,g;G,y)\right]\\
&+4\sqrt{
\frac{
u\ln(4u)\left(
\mathrm{Var}_{\mathrm{icv}}(\mathcal{S},\mathcal{Y})
+\frac{\tilde{\rho}^{\,2}}{\|C\|_2^{\,2}}
\right)
+\tilde{\rho}^{\,2}\ln\!\left(\frac{6n}{\delta}\right)
}{
(n-1)\tilde{\rho}^{\,2}
}
}\,,
\end{aligned}
\end{equation}
where $\mathrm{Var}_{\mathrm{i c v}}(\mathcal{S},\mathcal{Y}) = \frac{1}{|\mathcal{Y}|}\sum_{y\in \mathcal{Y}}\frac{1}{|I_y|}\sum_{G\in I_y}\bigl\|\mathbf{h}_G-\bm{\mu}_y\bigr\|_2^2$ represents the sum of the intra-class variance of graph-level representations, $I_y$ is a subset of $\mathcal{S}$ that graph samples belonging to the class $y$, and \(\bm{\mu}_y = \frac{1}{|I_y|} \sum_{G \in I_y} \mathbf{h}_G\) is the corresponding class cluster. \(\|C\|_2\) denotes the \(\ell_2\)-norm of model parameters of the classification layer \(C\) and
\(
\tilde{\rho}=\min\left\{\frac{\rho}{4\alpha},\ \frac{\rho'}{4\beta\lambda+2\alpha}\right\}.
\)
\end{theorem}
The detailed proof is provided in Appendix~\ref{appendix:proof theorem 3.1}. Theorem~\ref{thm:abstention_gen} shows that the generalization performance of AbstainGNN is upper bounded by two components: the empirical margin-based loss and the intra-class variance of graph-level representations. Therefore, reducing this intra-class variance is crucial for improving the performance of graph classification with abstention.

\subsection{Learning Objective}
To this end, we introduce a novel regularization loss to reduce the intra-class variance of graph-level representations: 
\begin{equation}
\label{intra-class_variance_loss}
\mathcal{L}_{\mathrm{icv}} = \frac{1}{|\mathcal{Y}|}\sum_{y\in\mathcal{Y}} \sum_{G \in I_y} w_G \|\mathbf{h}_G - \bm{\mu}_y\|^2,
\end{equation}
where \(w_G\) is the maximum logit of \(f(\mathbf A,\mathbf X) \). 
Based on $\mathcal{L}_{\mathrm{icv}}$, the total learning objective of AbstainGNN is 
\begin{equation}
\label{total_loss}
\mathcal{L}_{\mathrm{total}} \;=\; \mathcal{L}_{\mathrm{ce}}+\lambda_{r}\,\mathcal{L}_{\mathrm{icv}},
\end{equation}
where $\mathcal{L}_{\mathrm{ce}}$ is the standard cross-entropy loss for graph classification, corresponding to a smooth relaxation of the margin loss in the above generalization analysis, and $\lambda_{r}$ is a hyper-parameter for balancing $\mathcal{L}_{\mathrm{ce}}$ and $\mathcal{L}_{\mathrm{icv}}$. We theoretically analyze the optimization impact of $\mathcal{L}_{\mathrm{total}}$ for $\mathrm{Var}_{\mathrm{icv}}(\mathcal{S},\mathcal{Y})$: 


\begin{theorem}
\label{thm:intra_contraction}

Let \(\mathbf{h}_G^{(t)}\in\mathbb{R}^d\) and $\bm{\mu}_y^{(t)}$ denote the graph-level representation and the class cluster of $y$ at the iteration \(t\). We abbreviate $\mathrm{Var}_{\mathrm{icv}}(\mathcal{S},\mathcal{Y})$ at \(t\) as $M^{(t)}=\frac{1}{|\mathcal{Y}|}\sum_{y\in\mathcal{Y}}\frac{1}{|I_y|}\sum_{G\in I_y}\bigl\|\mathbf{h}_G^{(t)}-\bm{\mu}_y^{(t)}\bigr\|_2^2.$ Consider the following gradient-based updates: 
\begin{equation}
\mathbf{h}_G^{(t+1)}=\mathbf{h}_G^{(t)}-\eta \nabla_{\mathbf{h}_G^{(t)}}\mathcal{L}_{\mathrm{total}},
\qquad
\bm{\mu}_y^{(t+1)}=\frac{1}{|I_y|}\sum_{G\in I_y} \mathbf{h}_G^{(t+1)},
\end{equation}
where $\eta$ is the learning rate. Then $\mathrm{Var}_{\mathrm{icv}}(\mathcal{S},\mathcal{Y})$ has the following one-step recursion: 
\begin{equation}
\begin{aligned}
    M^{(t+1)} \;\le&\; \bigl(1-\frac{4\lambda_{r}\eta_t}{|\mathcal{Y}|\min_y|I_y|}
+\frac{12\lambda_{r}^2\eta_t^2}{\max_y|I_y|^2}\bigr)\,M^{(t)}\\
&+ O(\max_{G,t}\|\nabla_{\mathbf{h}_G^{(t)}}\mathcal{L}_{\mathrm{ce}}\|_2^2)\,\eta.
\end{aligned}
\end{equation}
If $0<\eta<\delta=\min\{\frac{\max_y|I_y|^2}{6|\mathcal{Y}|\lambda_{r}\min_y|I_y|},\frac{|\mathcal{Y}|\min_y|I_y|}{2\lambda_{r}}\}$ and $\rho_{\eta}=(1-\frac{2\lambda_{r}}{|\mathcal{Y}|\min_y|I_y|}\eta)\in(0,1)$, then we have 
\begin{equation}
\label{contraction process}
M^{(t)} \;\le\; \rho_{\eta}^{\,t}\,M^{(0)} + \frac{\eta^2}{1-\rho_{\eta}}\,O(\max_{G,t}\|\nabla_{\mathbf{h}_G^{(t)}}\mathcal{L}_{\mathrm{ce}}\|_2^2).
\end{equation}
In particular, \(M^{(t)}\) converges to an \(O(\eta)\)-neighborhood of zero.
\end{theorem}
The detailed proof is provided in Appendix~\ref{appendix:proof theorem 3.3}. Theorem~\ref{thm:intra_contraction} provides an optimization-based theoretic justification for introducing $\mathcal{L}_{\mathrm{icv}}$. Specifically, under the gradient descent on \(\mathcal{L}_{\mathrm{total}}\), $\mathrm{Var}_{\mathrm{icv}}(\mathcal{S},\mathcal{Y})$ satisfies a contraction inequality across iterations and converges to a tighter \(O(\eta)\)-neighborhood as the gradient of $\mathcal{L}_{\mathrm{ce}}$ diminishes during training. Consequently, minimizing $\mathcal{L}_{\mathrm{ce}}$ can reduce $\mathrm{Var}_{\mathrm{icv}}(\mathcal{S},\mathcal{Y})$, thereby tightening the resulting generalization bound.

\begin{figure}[ht]
    \centering
    \includegraphics[width=\linewidth]{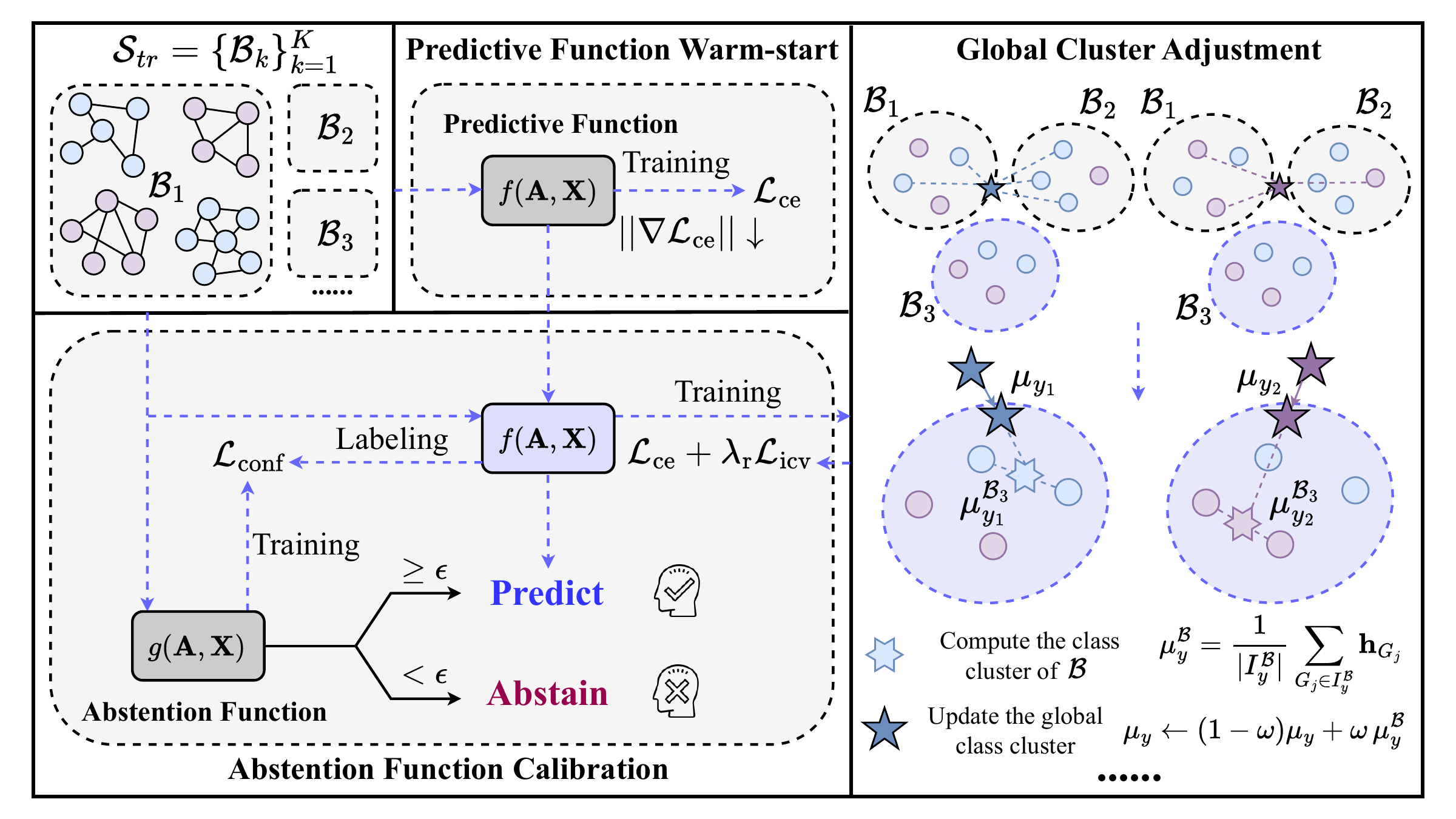}
    \caption{A simple model overview of AbstainGNN. It mainly consists of predictive function warm-start and abstention function calibration. The key operation in AbstainGNN is global cluster adjustment, which helps to more accurately estimate $\mathcal{L}_{\mathrm{icv}}$.}
    \label{fig:AbstainGNN}
\end{figure}

\subsection{Concrete Implementation}

Based on the analyses of Theorems~\ref{thm:abstention_gen} and~\ref{thm:intra_contraction}, we present a concrete implementation for optimizing the learning objective of AbstainGNN, which consists of two learning stages: \emph{Predictive Function Warm-start} and \emph{Abstention Function Calibration}. In the first stage, we employ a warm-start strategy for the predictive function $f$. Since Theorem~\ref{thm:intra_contraction} indicates that the reduction of $\mathrm{Var}_{\mathrm{i c v}}(\mathcal{S},\mathcal{Y})$ depends on the magnitude of $\|\nabla \mathcal{L}_{\mathrm{ce}}\|$. Therefore, we first pretrain $f$ with $\mathcal{L}_{\mathrm{ce}}$ to reach a small-gradient region, which accelerates the subsequent optimization of the generalization bound. In the second stage, we introduce global cluster adjustment to mitigate the estimation bias of $\bm{\mu}_y$ caused by batch-wise updates. Using the learned $f$, we then calibrate the abstention function $g$ so that its confidence scores are well-aligned with $f$'s predictions, achieving a favorable trade-off between coverage and risk for graph classification. Figure~\ref{fig:AbstainGNN} shows the model overview of AbstainGNN. 

\paragraph{Predictive Function Warm-start}
Given the training dataset $\mathcal{S}_{tr}=\{(G_i,y_i)\}_{i=1}^{|\mathcal{S}_{tr}|} \subset \mathcal{G}\times \mathcal{Y}$, we perform a warm-start strategy using $\mathcal{L}_{\mathrm{ce}}$ to drive the predictive function $f$ into a small-gradient region, yielding a stable initialization for the subsequent optimization. Following a batch-wise training manner, we first randomly shuffle $\mathcal{S}_{tr}$ and partition it into a collection of batches, denoted by $\mathcal{S}_{tr}=\{\mathcal{B}_k\}_{k=1}^{K}$ where $K$ is number of batches. Then, for each batch data $\mathcal{B}=\{(G_i,y_i)\}_{i=1}^{|\mathcal{B}|}$, we use the following batch-wise loss: 
\begin{equation}
\mathcal{L}_{\mathrm{ce}}(\mathcal{B})
= \frac{1}{|\mathcal{B}|}\sum_{i=1}^{|\mathcal{B}|}\ell_{ce}\big(\hat y_i,y_i\big),
\end{equation}
where $\ell_{\mathrm{ce}}(\cdot,\cdot)$ denotes the cross-entropy loss of the per-sample $G_i=(\mathbf{A}_i,\mathbf{X}_i)$, and $\hat y_i=\max_{y\in\mathcal{Y}}f_y(\mathbf{A}_i,\mathbf{X}_i)$ denotes the predicted class of $G_i$. $\mathcal{L}_{\mathrm{ce}}$ used in Eq.(\ref{total_loss}) is the cross-entropy loss computed over all batches sampled from $\mathcal{S}_{tr}$.

\paragraph{Abstention Function Calibration} To calculate the regularization loss in Eq.(\ref{intra-class_variance_loss}), we have to estimate the class cluster for each class in the representation space. A natural choice for graph classification is to estimate these class clusters via batch-wise updates. However, such updates are often biased due to random sampling and limited batchsizes, which in turn weakens the effectiveness of the intra-class variance regularization in Eq.(\ref{intra-class_variance_loss}). To address this issue, we introduce \emph{Global Cluster Adjustment} that maintains global class clusters smoothly throughout the entire training process. 
\par
Specifically, given $\mathcal{B}=\{(G_i,y_i)\}_{i=1}^{|\mathcal{B}|}$, we generate each graph-level representation $\mathbf{h}_{G_i}$ for  $\mathcal{B}$ according to Eq.(\ref{graph_representations}). For each class $y$, let $I_y^{\mathcal{B}}=\{G_i:y_i=y\}\subset\mathcal{B}$ that graph samples belonging to the class $y$. When $|I_y^{\mathcal{B}}|>0$, the class cluster of $\mathcal{B}$ is computed as
\begin{equation}
\bm{\mu}_y^{\mathcal{B}}
=
\frac{1}{|I_y^{\mathcal{B}}|}
\sum_{G_i\in I_y^{\mathcal{B}}}
\mathbf{h}_{G_i}.
\label{eq:batch_center}
\end{equation}
When $|I_y^{\mathcal{B}}|=0$, we simply skip the update for class $y$ in this iteration. We then update the global class cluster $\bm{\mu}_y$:
\begin{equation}
\bm{\mu}_y \leftarrow (1-\omega)\bm{\mu}_y + \omega\,\bm{\mu}_y^{\mathcal{B}},
\label{eq:ema_update}
\end{equation}
where $\omega\in(0,1)$ is a hyper-parameter to balance the contribution of historical information and newly observed data. Through this way, $\{\bm{\mu}_y\}_{y\in\mathcal{Y}}$ provides a global approximation of class clusters under the batch-wise training manner. Based on these updated clusters, we use the loss defined in Eq.(\ref{total_loss}) to optimize $f$. 
\par
We then train the graph-level abstention function $g$ to predict whether the current prediction of $f$ is correct on given graphs, thereby enabling abstention and coverage–risk evaluation. For each training sample $(G_i,y_i) \in \mathcal{B}$, we denote its predicted class as $\hat{y}_i$. A binary target label is introduced to indicate whether the prediction of $f$ is correct:
\begin{equation}
t_i=\mathbb{I}\big[\hat{y}_i=y_i\big]\in\{0,1\}.
\end{equation}
This formulation cast the confidence calibration of $g$ into  a binary classification, where $g$ is encouraged to assign high confidence scores to correctly predicted graphs and low confidence scores to misclassified ones. In AbstainGNN, $g$ is also implemented as a GNN following the calculation process described in Eq.(\ref{GNN_composite_map}--\ref{graph_representations}), the main difference is that the output of $g$ is mapped to a scalar representing the confidence score $\hat{t}_i$. The batch-wise loss of $\mathcal{B}$ for $g$ is defined as

\begin{equation}
\mathcal{L}_{\mathrm{conf}} (\mathcal{B})
=
\frac{1}{|\mathcal{B}|}
\sum_{(G_i,y_i)\in\mathcal{B}}\ell_{ce}\big(\hat{t}_i,t_i\big)
\end{equation}
The total loss used to optimize $g$ is $\mathcal{L}_{\mathrm{conf}}$, which is also a cross entropy loss evaluated over all batches drawn from $\mathcal{S}_{tr}$.

\begin{table*}[t]
\centering
\small
\caption{Classification Error Rate (risk) at different coverage levels (\%) on MUTAG, PROTEINS, and NCI1.}
\label{tab:chemical_graphs}
\begin{adjustbox}{width=0.9\textwidth}
\begin{tabular}{c|ccccccccc|c}
\toprule
{\footnotesize\diagbox[width=2cm,height=0.7cm]{\textbf{Method}}{\textbf{Coverage}}}
& \databar{10} & \databar{30} & \databar{50} & \databar{70}
& \databar{75} & \databar{80} & \databar{85} & \databar{90} & \databar{95} & \textbf{Rank}\\
\arrayrulecolor{black}\specialrule{0.08em}{0.2pt}{0.2pt}

\multicolumn{11}{>{\columncolor{gray!10}\centering}c}{\textbf{PROTEINS} \quad Classification Error Rate (\%) $\downarrow$} \\
\arrayrulecolor{black}\specialrule{0.08em}{0.2pt}{0.2pt}

SR
& \underline{3.4$\pm$5.1} & 10.7$\pm$3.0 &  \underline{14.5$\pm$0.7} & \underline{17.9$\pm$1.1}
& 19.6$\pm$1.7 & 20.7$\pm$1.6 & 21.7$\pm$1.3 & 23.5$\pm$1.1 & 23.9$\pm$0.9 & \underline{2.8}\\

MC-Dropout
& 5.2$\pm$4.3 & 15.2$\pm$4.7 & 18.8$\pm$2.3 & 21.3$\pm$1.6
& 21.4$\pm$1.5 & 21.6$\pm$1.8 & 22.2$\pm$1.6 & \underline{22.1$\pm$1.7} & \underline{23.0$\pm$1.6} & 4.6\\

Deep Gamblers
& 7.8$\pm$5.1 & 13.4$\pm$4.1 & 19.5$\pm$1.8 & 23.6$\pm$2.8
& 25.4$\pm$2.1 & 25.8$\pm$2.4 & 25.8$\pm$2.3 & 25.5$\pm$2.1 & 24.7$\pm$2.4 & 7.3\\

SAT
& 6.1$\pm$5.0 & 12.8$\pm$3.7 & 17.5$\pm$1.5 & 23.5$\pm$1.9
& 24.2$\pm$1.7 & 25.2$\pm$1.8 & 25.5$\pm$1.6 & 25.4$\pm$2.0 & 24.7$\pm$2.1 & 6.1\\

CCL-SC
& 3.5$\pm$5.1 & \underline{10.4$\pm$3.0} & 14.8$\pm$0.7 & 18.2$\pm$0.3
& \underline{19.5$\pm$1.2} & \underline{20.3$\pm$1.6} & \underline{21.4$\pm$1.4} & 23.4$\pm$0.8 & 23.7$\pm$1.1 & 2.7\\

NCwR
& 9.6$\pm$3.6 & 18.2$\pm$2.7 & 26.1$\pm$4.0 & 26.9$\pm$3.1
& 26.9$\pm$2.2 & 25.9$\pm$2.1 & 25.5$\pm$2.5 & 24.8$\pm$2.5 & 24.5$\pm$2.1 & 7.1\\

GraphPPD
& 4.5$\pm$3.2 & 11.0$\pm$4.6 & 15.4$\pm$1.8 & 20.1$\pm$2.1 
& 21.0$\pm$2.1 & 21.5$\pm$1.9 & 23.2$\pm$1.1 & 24.5$\pm$1.1 & 24.7$\pm$1.3 & 4.4\\

\arrayrulecolor{black}\specialrule{0.03em}{0.2pt}{0.2pt}
\textbf{AbstainGNN}
& \textbf{1.8$\pm$3.6} & \textbf{9.5$\pm$4.1} & \textbf{12.9$\pm$1.2}
& \textbf{17.4$\pm$1.0} & \textbf{18.2$\pm$1.0} & \textbf{19.3$\pm$1.3}
& \textbf{20.5$\pm$0.6} & \textbf{21.7$\pm$1.0} & \textbf{22.7$\pm$1.2} & \textbf{1.0} \\

\arrayrulecolor{black}\specialrule{0.08em}{0.2pt}{0.2pt}

\multicolumn{11}{>{\columncolor{gray!10}\centering}c}{\textbf{MUTAG} \quad Classification Error Rate (\%) $\downarrow$} \\
\arrayrulecolor{black}\specialrule{0.08em}{0.2pt}{0.2pt}

SR
&  \textbf{0.0$\pm$0.0} & 1.7$\pm$3.7 & 8.4$\pm$2.9 & 13.8$\pm$2.1
& 15.7$\pm$3.2 & 17.3$\pm$4.4 & 16.3$\pm$4.0 & 17.1$\pm$4.0 & 17.8$\pm$5.4 & 4.1 \\

MC-Dropout
& \textbf{0.0$\pm$0.0} & 8.3$\pm$5.3 & 14.7$\pm$5.2 & 15.4$\pm$5.4
& 15.0$\pm$5.2 & 16.7$\pm$3.7 & 16.9$\pm$4.2 & 17.6$\pm$4.6 & 17.2$\pm$4.8 & 5.6 \\

Deep Gamblers
&  \textbf{0.0$\pm$0.0} &  \underline{1.7$\pm$3.3} &  \underline{5.3$\pm$3.3} & 15.4$\pm$4.9
& 17.1$\pm$3.5 & 18.0$\pm$4.5 & 17.5$\pm$4.7 & 17.6$\pm$4.6 & 18.3$\pm$3.8 & 5.2 \\

SAT
&  \textbf{0.0$\pm$0.0} &  \underline{1.7$\pm$3.3} & 7.4$\pm$2.6 & 16.9$\pm$5.8
& 17.1$\pm$5.2 & 18.0$\pm$5.0 & 18.1$\pm$5.4 & 19.4$\pm$4.8 & 18.9$\pm$4.4 & 5.8 \\

CCL-SC
&  \textbf{0.0$\pm$0.0} &  \underline{1.7$\pm$3.3} & 7.4$\pm$4.2 & 13.1$\pm$3.1
& 16.4$\pm$3.6 & 16.7$\pm$3.7 & 17.5$\pm$3.8 & 17.1$\pm$3.4 & \underline{16.7$\pm$3.9} & \underline{2.9} \\

NCwR
&  \textbf{0.0$\pm$0.0} & 3.3$\pm$6.7 & 9.5$\pm$3.9 & \underline{11.5$\pm$3.4}
& \underline{11.4$\pm$2.7} & \underline{12.0$\pm$1.6} & \underline{13.1$\pm$2.3} & \underline{15.3$\pm$3.9} & 17.2$\pm$4.8 & \underline{2.9} \\

GraphPPD
&  \textbf{0.0$\pm$0.0} & 3.6$\pm$5.0 & 7.4$\pm$2.9 & 14.8$\pm$2.7
& 15.0$\pm$4.7 & 16.7$\pm$4.7 & 17.5$\pm$3.6 & 17.6$\pm$3.6 & 17.8$\pm$4.6 & 4.2 \\

\arrayrulecolor{black}\specialrule{0.03em}{0.2pt}{0.2pt}
\textbf{AbstainGNN}
& \textbf{0.0$\pm$0.0} & \textbf{0.0$\pm$0.0} & \textbf{5.2$\pm$0.1} & \textbf{10.4$\pm$2.7}
& \textbf{10.7$\pm$1.5} & \textbf{11.5$\pm$1.8} & \textbf{12.5$\pm$2.1} & \textbf{13.3$\pm$3.2} & \textbf{14.8$\pm$2.7} & \textbf{1.0} \\

\arrayrulecolor{black}\specialrule{0.08em}{0.2pt}{0.2pt}

\multicolumn{11}{>{\columncolor{gray!10}\centering}c}{\textbf{NCI1} \quad Classification Error Rate (\%) $\downarrow$} \\
\arrayrulecolor{black}\specialrule{0.08em}{0.2pt}{0.2pt}

SR
& 10.6$\pm$5.2 & 16.1$\pm$2.2 & 22.5$\pm$1.3 & 25.8$\pm$1.4
& 26.9$\pm$1.5 & 27.9$\pm$1.7 & 28.6$\pm$1.4 & 29.5$\pm$1.3 & 30.2$\pm$1.6 & 3.7 \\

MC-Dropout
& 13.7$\pm$4.1 & 21.1$\pm$3.2 & 25.4$\pm$2.8 & 28.2$\pm$1.7
& 28.5$\pm$1.6 & 29.3$\pm$1.6 & 29.7$\pm$1.7 & 29.9$\pm$1.6 & 30.5$\pm$1.8 & 5.9 \\

Deep Gamblers
& 13.5$\pm$3.3 & 22.7$\pm$2.7 & 27.7$\pm$2.1 & 29.7$\pm$1.6
& 30.0$\pm$1.8 & 30.8$\pm$2.2 & 31.3$\pm$2.2 & 31.4$\pm$2.0 & 30.8$\pm$2.1 & 7.7 \\

SAT
& 9.6$\pm$4.2 & 19.5$\pm$3.2 & 26.4$\pm$2.2 & 28.6$\pm$1.4
& 29.4$\pm$1.2 & 29.8$\pm$1.1 & 30.4$\pm$1.2 & 30.6$\pm$1.3 & 30.7$\pm$1.1 & 5.4 \\

CCL-SC
& 11.1$\pm$5.1 & 15.6$\pm$1.6 & 22.4$\pm$1.1 & 26.4$\pm$1.7
& 27.4$\pm$1.5 & 28.2$\pm$1.3 & 29.0$\pm$1.4 & 29.6$\pm$1.3 & 30.3$\pm$1.5 & 3.6 \\

NCwR
& 12.3$\pm$3.7 & 23.3$\pm$2.5 & 27.7$\pm$1.7 & 29.7$\pm$2.0
& 30.1$\pm$1.9 & 30.2$\pm$2.1 & 30.5$\pm$1.7 & 30.7$\pm$1.5 & 30.5$\pm$1.6 & 6.8 \\

GraphPPD
& \underline{9.5$\pm$5.5} & \underline{15.3$\pm$3.1} & \underline{21.3$\pm$1.2} & \underline{25.2$\pm$1.5} 
& \underline{26.5$\pm$1.3} & \underline{27.2$\pm$1.5} & \underline{28.0$\pm$1.3} & \underline{29.0$\pm$1.1} & \underline{29.9$\pm$1.1} & \underline{2.0} \\

\arrayrulecolor{black}\specialrule{0.03em}{0.2pt}{0.2pt}
\textbf{AbstainGNN}
& \textbf{8.5$\pm$4.1} & \textbf{15.1$\pm$2.4} & \textbf{20.9$\pm$2.1} & \textbf{24.2$\pm$1.5}
& \textbf{24.9$\pm$1.6} & \textbf{25.7$\pm$1.5} & \textbf{26.8$\pm$1.8} & \textbf{27.6$\pm$1.4} & \textbf{28.4$\pm$1.6} & \textbf{1.0}\\

\bottomrule
\end{tabular}
    \end{adjustbox}
\end{table*}

\subsection{Computational complexity}
We analyze the computational complexity of AbstainGNN. We denote by $N$ and $E$ the average numbers of nodes and edges in graphs within a batch, respectively. AbstainGNN follows a two-stage training framework. In the first stage, the computation is dominated by the optimization of the predictive function $f$, resulting in a computational complexity of $\mathcal{O}(E d + N d^2)$. In the second stage, the training process consists of three main components. First, computing batch-wise class clusters incurs a cost of $\mathcal{O}(|\mathcal{B}| d)$ per batch. Second, computing the regularization loss (Eq.(\ref{intra-class_variance_loss})) require $\mathcal{O}(|\mathcal{B}| d)$; Third, training the graph-level abstention function $g$ has a computational complexity of $\mathcal{O}(E d + N d^2)$ per iteration. Aggregating these costs, AbstainGNN maintains an acceptable computational complexity of $\mathcal{O}((E+|\mathcal{B}|) d + N d^2)$. An empirical comparison of training runtime is provided in ~\ref{sec:Runtime-Comparison}. 


\begin{table}[h]
\centering
\caption{Statistics of used datasets.}
\label{tab:dataset_statistics}
\begin{tabular}{l c c c c}
\hline
Dataset & Graphs & Classes & Avg. nodes & Avg. edges \\
\hline
PROTEINS     & 1,113 & 2 & 39.1  & 72.8   \\
MUTAG        & 188   & 2 & 17.9  & 19.8   \\
NCI1         & 4,110 & 2 & 29.9  & 32.3   \\
IMDB-BINARY  & 1,000 & 2 & 19.8  & 96.5   \\
COLLAB       & 5,000 & 3 & 74.5  & 2457.8 \\
\hline
\end{tabular}
\end{table}

\section{Experiment}

\subsection{Experimental Setup}
\paragraph{Datasets} We conduct experiments on five widely used benchmark datasets for graph classification~\cite{morris2020tudataset}, including PROTEINS, MUTAG, NCI1, IMDB-BINARY, and COLLAB. Among them, MUTAG, PROTEINS, and NCI1 are chemical graph datasets, while IMDB-BINARY and COLLAB are social graph datasets. Following previous works~\cite{errica2019fair,bodnar2021weisfeiler}, we use the provided node attributes for chemical graph datasets, where each node is represented by a one-hot encoding of its atom type. For social graph datasets, where nodes do not have inherent attributes, we adopt node degree as the node attribute. The statistics of these datasets are summarized in Table ~\ref{sec:Dataset_Statistics}.   

\label{sec:Dataset_Statistics}
We use five benchmark datasets of graph classification, and the statistics of them are provided in Table~\ref{tab:dataset_statistics}.

\begin{table*}[!t]
\centering
\small
\caption{Classification Error Rate (risk) at different coverage levels (\%) on IMDB-BINARY and COLLAB.}
\label{tab:social_graphs}
\begin{adjustbox}{width=0.9\textwidth}
\begin{tabular}{c|ccccccccc|c}
\toprule

{\footnotesize\diagbox[width=2cm,height=0.7cm]{\textbf{Method}}{\textbf{Coverage}}}
& \databar{10} & \databar{30} & \databar{50} & \databar{70}
& \databar{75} & \databar{80} & \databar{85} & \databar{90} & \databar{95} & \textbf{Rank}\\

\arrayrulecolor{black}\specialrule{0.08em}{0.2pt}{0.2pt}
\multicolumn{11}{>{\columncolor{gray!10}\centering}c}{\textbf{IMDB-BINARY} \quad Classification Error Rate (\%) $\downarrow$} \\
\arrayrulecolor{black}\specialrule{0.08em}{0.2pt}{0.2pt}

SR
& \textbf{0.0$\pm$0.0} & 6.9$\pm$5.4 & 13.2$\pm$3.6 & 18.6$\pm$3.8
& 20.5$\pm$2.8 & 22.0$\pm$3.0 & 23.7$\pm$3.2 & 25.0$\pm$3.2 & 25.8$\pm$2.9 & 4.2 \\

MC-Dropout
& \textbf{0.0$\pm$0.0} & 11.3$\pm$5.1 & 19.8$\pm$2.1 & 22.7$\pm$1.7
& 23.5$\pm$2.0 & 23.1$\pm$1.9 & 23.6$\pm$2.1 & 24.3$\pm$1.5 & 25.1$\pm$2.0 & 5.7\\

Deep Gamblers
& \underline{1.0$\pm$2.2} & 10.9$\pm$2.0 & 17.4$\pm$2.4 & 22.0$\pm$2.2
& 23.0$\pm$2.3 & 23.6$\pm$2.5 & 24.2$\pm$3.0 & \underline{24.2$\pm$2.8} & \underline{24.6$\pm$2.8} & 5.6\\

SAT
& \textbf{0.0$\pm$0.0} & 7.2$\pm$2.3 & 14.7$\pm$2.7 & 19.7$\pm$3.3
& 20.8$\pm$2.8 & 22.4$\pm$2.0 & \underline{23.2$\pm$2.4} & 24.6$\pm$3.2 & 25.3$\pm$2.7 & 4.7 \\

CCL-SC
& \textbf{0.0$\pm$0.0} & \underline{6.9$\pm$4.8} & 13.4$\pm$3.9 & 18.7$\pm$3.5
& 20.8$\pm$2.6 & 22.1$\pm$2.8 & 23.2$\pm$2.7 & 25.1$\pm$2.9 & 25.5$\pm$2.5 & 3.9 \\

NCwR
& \textbf{0.0$\pm$0.0} & 7.2$\pm$2.3 & 13.8$\pm$2.9 & 20.1$\pm$2.4
& 21.5$\pm$2.2 & 21.8$\pm$1.6 & 23.8$\pm$1.4 & 24.3$\pm$2.2 & 25.1$\pm$2.3 & 4.6 \\

GraphPPD
& 2.0$\pm$2.4 & 9.7$\pm$3.4 & \textbf{12.4$\pm$3.6} & \underline{18.3$\pm$3.0}
& \underline{20.1$\pm$3.7} & \underline{21.4$\pm$3.0} & 23.3$\pm$2.3 & 24.7$\pm$3.5 & 25.2$\pm$3.1 & \underline{3.6}\\

\arrayrulecolor{black}\specialrule{0.03em}{0.2pt}{0.2pt}

\textbf{AbstainGNN}
& \textbf{0.0$\pm$0.0} & \textbf{6.3$\pm$3.1} & \underline{13.0$\pm$3.6} & \textbf{17.1$\pm$1.9}
& \textbf{18.9$\pm$2.6} & \textbf{19.9$\pm$2.4} & \textbf{21.8$\pm$2.4} & \textbf{22.9$\pm$2.7} & \textbf{23.3$\pm$2.2} & \textbf{1.1} \\


\arrayrulecolor{black}\specialrule{0.08em}{0.2pt}{0.2pt}
\multicolumn{11}{>{\columncolor{gray!10}\centering}c}{\textbf{COLLAB} \quad Classification Error Rate (\%) $\downarrow$} \\
\arrayrulecolor{black}\specialrule{0.08em}{0.2pt}{0.2pt}

SR
& \textbf{0.0$\pm$0.0} & 1.9$\pm$0.7 & 5.4$\pm$0.6 & \underline{11.6$\pm$0.7}
& \underline{13.2$\pm$0.5} & \underline{14.6$\pm$0.4} & \underline{16.3$\pm$0.5} & \underline{17.7$\pm$0.5} & \underline{18.8$\pm$0.4} & \underline{2.2} \\

MC-Dropout
& \textbf{0.0$\pm$0.0} & 1.5$\pm$0.5 & 8.5$\pm$0.3 & 15.5$\pm$0.5
& 16.6$\pm$0.4 & 17.6$\pm$0.7 & 18.4$\pm$0.7 & 19.1$\pm$0.8 & 19.7$\pm$0.6 & 6.0 \\

Deep Gamblers
& \textbf{0.0$\pm$0.0} & 4.7$\pm$1.8 & 9.3$\pm$0.8 & 12.7$\pm$1.0
& 13.9$\pm$0.9 & 15.3$\pm$1.1 & 16.7$\pm$0.8 & 17.8$\pm$0.7 & 19.2$\pm$0.7 & 4.7 \\

SAT
& \textbf{0.0$\pm$0.0} & 2.8$\pm$1.3 & 6.5$\pm$1.6 & 12.7$\pm$1.0
& 14.5$\pm$0.9 & 16.0$\pm$0.8 & 17.2$\pm$0.8 & 18.4$\pm$0.7 & 19.5$\pm$0.6 & 5.4 \\

CCL-SC
& \textbf{0.0$\pm$0.0} & 2.1$\pm$0.5 & 6.0$\pm$1.3 & 11.8$\pm$1.2
& 13.4$\pm$1.3 & 15.1$\pm$1.1 & 16.6$\pm$0.8 & 17.8$\pm$0.7 & 19.3$\pm$0.9 & 3.2 \\

NCwR
& \underline{5.8$\pm$2.0} & 10.5$\pm$3.4 & 14.1$\pm$3.0 & 16.7$\pm$1.8
& 17.3$\pm$1.4 & 17.9$\pm$1.3 & 18.5$\pm$0.9 & 18.9$\pm$0.7 & 19.3$\pm$0.5 & 6.9 \\

GraphPPD
& \textbf{0.0$\pm$0.0} & \underline{1.5$\pm$0.3} & \underline{4.8$\pm$0.7} &12.2$\pm$0.6
& 13.8$\pm$0.4 & 15.6$\pm$0.7 & 16.7$\pm$0.6 & 18.2$\pm$0.6 & 19.5$\pm$0.7 & 3.6 \\

\arrayrulecolor{black}\specialrule{0.03em}{0.2pt}{0.2pt}

\textbf{AbstainGNN}
& \textbf{0.0$\pm$0.0}  & \textbf{1.4$\pm$0.4} & \textbf{4.0$\pm$0.7} & \textbf{10.8$\pm$0.8}
& \textbf{12.1$\pm$0.5} & \textbf{13.9$\pm$0.8} & \textbf{15.7$\pm$0.4} & \textbf{17.4$\pm$0.5} & \textbf{18.0$\pm$0.4} & \textbf{1.0} \\

\bottomrule
\end{tabular}
\end{adjustbox}
\end{table*}

\paragraph{Compared Baselines} We describe compared baselines as follows,
\begin{itemize}

    \item \textbf{SR} \citep{geifman2017selective}: It is one of the most classic algorithms, which directly uses the maximum logit from the softmax layer as a confidence score to decide whether to abstain from making predictions. 

    \item \textbf{MC-Dropout} \citep{gal2016dropout}: Instead of relying on the maximum logit, this approach keeps dropout active at test time and performs multiple stochastic forward passes to estimate predictive uncertainty, which is then used as a confidence score for abstention.

    \item \textbf{Deep Gamblers} \citep{liu2019deep}: This method extends the original classification task by introducing an additional abstention class, enabling the model to predict the confidence of abstaining from making predictions. 

    \item \textbf{SAT} \citep{huang2020self}: Similar to Deep Gamblers, this method also introduces the abstention class. The key difference is that SAT can leverages model predictions during training to adaptively modulate each sample's contribution. 

    \item \textbf{CCL-SC} \citep{wu2024cclsc}: It is a competitive coverage-based method that leverages contrastive learning to aggregate the representations of samples in the same category, leading to improved performance under coverage constraints.
    
    \item \textbf{NCwR} \citep{kuchipudinode}: This is the first work that introduces abstention into node classification. It proposes both cost-based and coverage-based formulations, and we adopt the cost-based variant due to its superior empirical performance.

    \item \textbf{GraphPPD} \citep{pal2025graphppd}: It proposes a variational modeling framework to estimate the uncertainty as the rejection criterion.

\end{itemize}

\paragraph{Implementation Details}
\label{sec:implementation_details}
Since the first six baselines are not designed for graph-level tasks, we adopt the same GNN backbone as the predictive function for both our model and these baselines, following the optimal hyper-parameter settings reported in their original literature. For GraphPPD, we reproduce its proposed feature extractor and posterior approximation as described in the original paper. In our model, the predictive function $f$ is implemented using a simple GNN architecture consisting of 3 GCNConv layers~\cite{kipf2016gcn} as message-passing layers, with sum pooling as the readout operator. The graph-level abstention function $g$ shares the same GNN framework as $f$, except that it contains only 1 GCNConv layer. For all datasets, the training–testing split ratio is fixed at 0.8:0.2. We conduct experiments with five different random seeds and report the mean and variance of empirical results.
\par
To ensure a fair comparison, the dimensionality of the learned graph-level representations is fixed to 64 for all methods. The balanced factor $\omega$ in Eq.(\ref{eq:ema_update}) is fixed as 0.5. We use grid search to tune the remaining hyper-parameters. The search spaces are given here: $\lambda_{r}$ in Eq.(\ref{total_loss}) $\in \{0.1, 0.3, 0.5, 0.7, 0.9\}$, the number of warm-start epochs $T_{ws} \in \{ 50, 100, 150,175, 200\}$, and the batchsize $|\mathcal{B}| \in \{16, 32, 64, 128, 256\}$. In addition, the number of epochs in the second stage is fixed to 50, and an early stopping strategy is applied. Model training is terminated when the sum of the losses of $f$ and $g$ no longer decrease for three consecutive epochs.

\paragraph{Metrics} In classification with rejection, two key quantities are coverage and risk. Following the standard evaluation protocol~\cite{geifman2017selective,geifman2019selectivenet,charoenphakdee2021classification}, we first rank all test samples using the abstention function and then select the top-ranked samples for prediction at different coverage levels. Specifically, the coverage levels are varied from \{10\%, 30\%, 50\%, 70\%, 75\%, 80\%, 85\%, 90\%, 95\%\}. This design choice is motivated by the fact that, at low coverage levels, most methods can easily select high-confidence samples, which does not accurately reflect the true model capability. Therefore, we adopt a finer-grained partition at higher coverage levels. Risk is evaluated by the classification error rate.




\begin{table*}[t]
\centering
\small
\caption{Performance Comparisons of different model variants of AbstainGNN on PROTEINS and IMDB-BINARY.}
\label{tab:variant}
\begin{adjustbox}{width=0.9\textwidth}
\begin{tabular}{c|ccccccccc}
\toprule

{\footnotesize\diagbox[width=2.8cm,height=0.6cm]{\textbf{Method}}{\textbf{Coverage}}}
& \databar{10} & \databar{30} & \databar{50} & \databar{70}
& \databar{75} & \databar{80} & \databar{85} & \databar{90} & \databar{95} \\

\arrayrulecolor{black}\specialrule{0.08em}{0.2pt}{0.2pt}
\multicolumn{10}{>{\columncolor{gray!10}\centering}c}{\textbf{PROTEINS} \quad Classification Error Rate (\%) $\downarrow$} \\
\arrayrulecolor{black}\specialrule{0.08em}{0.2pt}{0.2pt}

AbstainGNN-GAT
& 4.5$\pm$4.5 & \textbf{9.2$\pm$4.6} & 15.3$\pm$2.2 & 21.5$\pm$2.2
& 21.8$\pm$2.2 & 21.7$\pm$2.0 & 22.1$\pm$1.8 & 22.3$\pm$1.8 & \textbf{22.5$\pm$2.0} \\

AbstainGNN-MLP
& 7.1$\pm$6.0 & 12.2$\pm$4.9 & 18.0$\pm$1.6 & 24.4$\pm$2.3 
& 25.2$\pm$2.4 & 25.5$\pm$2.0 & 25.7$\pm$1.9 & 25.7$\pm$2.2 & 25.0$\pm$2.0 \\

AbstainGNN-BC
& 6.1$\pm$6.6 & 12.5$\pm$5.1 & 17.7$\pm$3.3 & 23.2$\pm$2.7 
& 23.8$\pm$2.7 & 24.3$\pm$2.4 & 24.3$\pm$3.0 & 24.4$\pm$2.9 & 24.1$\pm$3.1 \\

\arrayrulecolor{black}\specialrule{0.03em}{0.2pt}{0.2pt}
\textbf{AbstainGNN}
& \textbf{1.8$\pm$3.6} & 9.5$\pm$4.1 & \textbf{12.9$\pm$1.2}
& \textbf{17.4$\pm$1.0} & \textbf{18.2$\pm$1.0} & \textbf{19.3$\pm$1.3}
& \textbf{20.5$\pm$0.6} & \textbf{21.7$\pm$1.0} & 22.7$\pm$1.2 \\

\arrayrulecolor{black}\specialrule{0.08em}{0.2pt}{0.2pt}
\multicolumn{10}{>{\columncolor{gray!10}\centering}c}{\textbf{IMDB-BINARY} \quad Classification Error Rate (\%) $\downarrow$} \\
\arrayrulecolor{black}\specialrule{0.08em}{0.2pt}{0.2pt}

AbstainGNN-GAT
& 3.0$\pm$2.5 & 11.7$\pm$3.0 & 15.0$\pm$1.8 & 19.3$\pm$2.1 
& 20.5$\pm$2.6 & 21.6$\pm$3.3 & 23.1$\pm$2.9 & 23.6$\pm$3.1 & 24.4$\pm$2.9 \\

AbstainGNN-MLP
& 8.4$\pm$3.8 & 16.9$\pm$3.2 & 21.4$\pm$4.2 & 23.9$\pm$3.7 
& 24.1$\pm$2.9 & 24.6$\pm$2.0 & 24.9$\pm$2.1 & 25.9$\pm$1.7 & 26.6$\pm$1.8 \\

AbstainGNN-BC
& 3.0$\pm$4.0 & 11.3$\pm$3.7 & 16.6$\pm$2.6 & 20.9$\pm$2.5 
& 22.0$\pm$2.4 & 22.8$\pm$2.5 & 23.5$\pm$1.6 & 23.9$\pm$2.4 & 25.3$\pm$1.6 \\

\arrayrulecolor{black}\specialrule{0.03em}{0.2pt}{0.2pt}

\textbf{AbstainGNN}
& \textbf{0.0$\pm$0.0} & \textbf{6.3$\pm$3.1} & \textbf{13.0$\pm$3.6} & \textbf{17.1$\pm$1.9}
& \textbf{18.9$\pm$2.6} & \textbf{19.9$\pm$2.4} & \textbf{21.8$\pm$2.4} & \textbf{22.9$\pm$2.7} & \textbf{23.3$\pm$2.2} \\

\bottomrule
\end{tabular}
\end{adjustbox}
\end{table*}

\begin{figure*}[t]
    \centering
    \includegraphics[width=\linewidth]{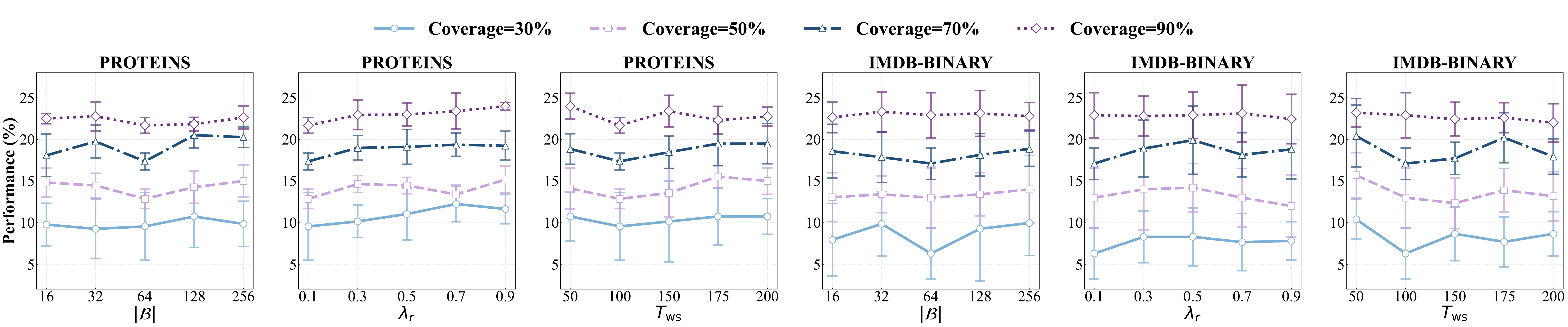}
    \caption{Hyperparameter Sensitivity Analysis of  AbstainGNN.}
    \label{fig:hyperparameter}
\end{figure*}

\subsection{Performance Comparison}

Table~\ref{tab:chemical_graphs} and Table~\ref{tab:social_graphs} report the risk comparison under different coverage levels on chemical graph datasets and social graph datasets, respectively. The best and second-best results are highlighted in bold and underlined. From these results, we can draw the following conclusions. \textbf{(1)} Across all datasets, AbstainGNN demonstrates clear improvements across different coverage levels. In particular, AbstainGNN yields an average relative risk reduction of 9.8\% on PROTEINS, 16.8\% on MUTAG, 4.8\% on NCI1, 15.8\% on IMDB-BINARY, and 5.9\% on COLLAB. \textbf{(2)} At low coverage levels, high-confidence samples are relatively easy to identify. As a result, on MUTAG, IMDB-BINARY, and COLLAB, many methods achieve zero risk when the coverage is set to 10\%. However, as the coverage increases, especially beyond 80\%, our method is still able to select reliable decision samples, leading to consistently superior performance. \textbf{(3)} The baselines show large performance discrepancies for different datasets. For instance, NCwR achieves competitive results on MUTAG, but performs moderately on social graph datasets. Compared with these baselines, our model exhibits superior generalization and consistently achieves the best overall performance across all datasets.

\begin{figure}[h]
    \centering
    \includegraphics[width=\linewidth]{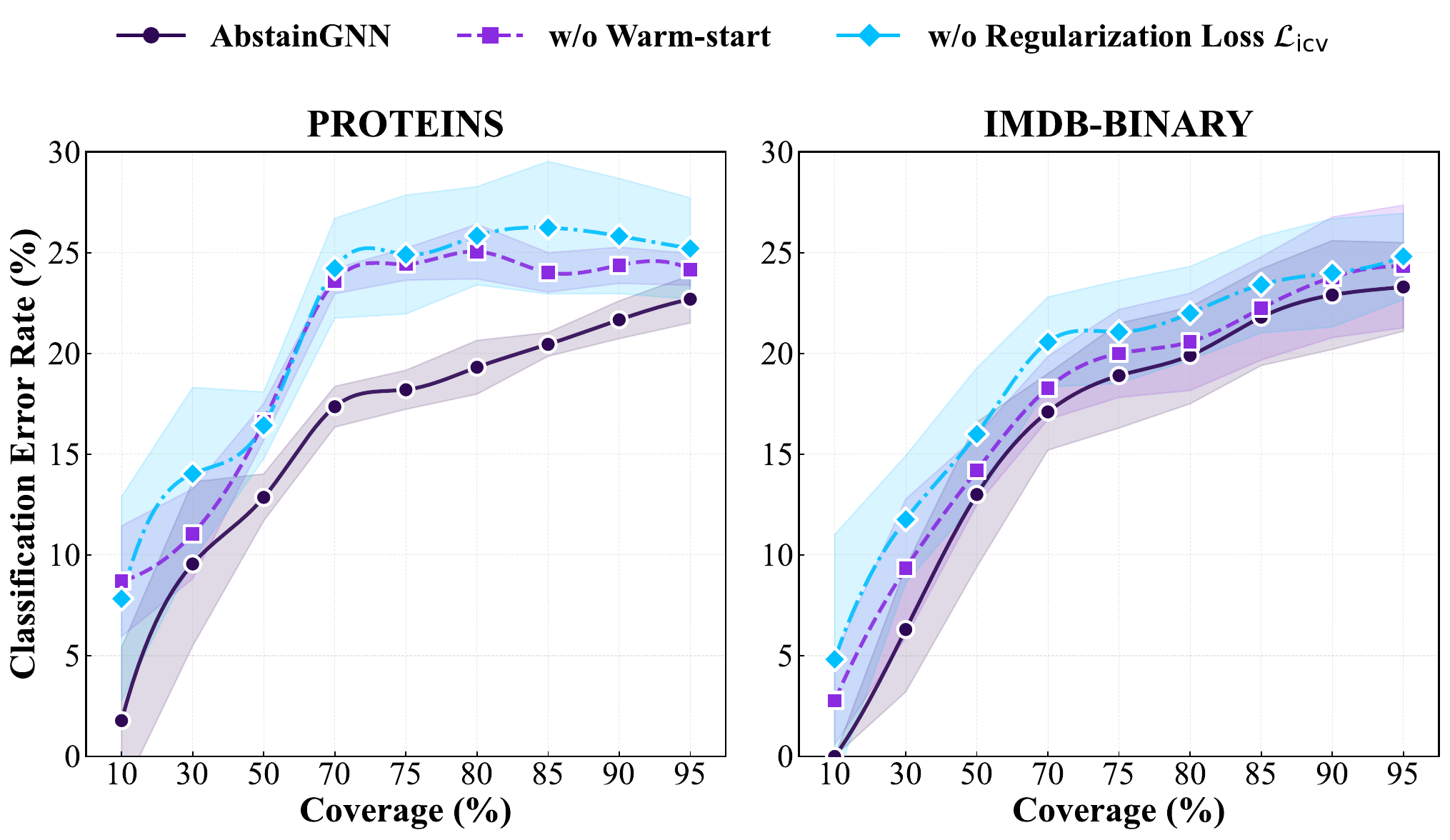}
    \caption{Ablation Study of AbstainGNN.}
    \label{fig:ablation}
\end{figure}

\subsection{Ablation Study} 
Figure~\ref{fig:ablation} presents an ablation study of AbstainGNN on PROTEINS and IMDB-BINARY across different coverage levels. `w/o Warm-start' refers to the one that removes the first stage of predictive function warm-start and only performs the second stage of abstention function calibration. `w/o Regularization Loss $\mathcal{L}_{\mathrm{icv}}$' denotes the one that excludes the intra-class variance regularization in Eq.(\ref{intra-class_variance_loss}), optimizing $f$ and $g$ solely with the cross-entropy loss. As shown in Figure~\ref{fig:ablation}, removing either component consistently degrades the risk–coverage trade-off, with the gap becoming more obvious at medium-to-high coverage levels. On PROTEINS, AbstainGNN achieves the lowest risk across all coverages, and its advantage grows substantially once the coverage exceeds 50\%. On IMDB-BINARY, we observe a similar but relatively milder trend. Overall, these results confirm the benefits of warm-start and $\mathcal{L}_{\mathrm{icv}}$: warm-start stabilizes the optimization process, while $\mathcal{L}_{\mathrm{icv}}$ strengthens the regularization of graph-level representations, which is particularly important at higher coverage levels.


\subsection{Model Variants}
In this section, we design three model variants: AbstainGNN-GAT, AbstainGNN-MLP, and AbstainGNN-BC. AbstainGNN-GAT replaces GCNConv layers with graph attention networks (GAT)~\cite{velivckovic2017graph} to model the abstention function $g$. AbstainGNN-MLP employs a multilayer perceptron (MLP) to model $g$, ignoring graph structural information. AbstainGNN-BC removes global cluster adjustment and instead relies on batch-wise class clusters to optimize $f$. Table~\ref{tab:variant} reports their comparisons on PROTEINS and IMDB-BINARY. From the results, we make the following observations. \textbf{(1)} AbstainGNN-MLP performs substantially worse than both AbstainGNN-GAT and AbstainGNN, indicating that the ability of $g$ to capture graph structural information is crucial. Although AbstainGNN-GAT slightly outperforms AbstainGNN on a few coverage levels on PROTEINS, AbstainGNN achieves the best overall performance. In addition, AbstainGNN-GAT incurs higher computational complexity than AbstainGNN, suggesting that designing more suitable GNN architectures for $g$ is a promising direction for future work. \textbf{(2)} Compared with AbstainGNN-BC, AbstainGNN consistently achieves superior results, showing that optimizing $\mathcal{L}_{\mathrm{icv}}$ based on batch-wise class clusters is insufficient. Maintaining global class clusters, as adopted in AbstainGNN, is a more effective strategy.


\subsection{Hyper-parameter Analysis}
Figure~\ref{fig:hyperparameter} analyzes the impact of three main hyper-parameters: $\lambda_{r}$,  $T_{ws}$, and $|\mathcal{B}|$, as described in Section~\ref{sec:implementation_details}. Several observations can be drawn from this figure. \textbf{(1)} Setting $\lambda_{r}=0.1$ leads to the best overall performance. This choice effectively reduces intra-class variance to improve confidence ranking, while avoiding overly compact graph-level representations that would reduce effective margins and degrade performance at high coverage levels. \textbf{(2)} The optimal warm-start epoch is $T_{\mathrm{ws}}=100$. This warm-start strategy aims to reduce the magnitude of the cross-entropy gradients. After sufficient warm-up, the updates of $\mathcal{L}_{\mathrm{ce}}$ become more stable, enabling $\mathcal{L}_{\mathrm{icv}}$ to be effectively optimized.     \textbf{(3)} A batchsize of $|\mathcal{B}|=64$ achieves the best trade-off between model performance and training stability. The small batchsizes yield noisy estimates of the batch-wise class clusters, and an excessively large batchsize can adversely affect the optimization process, potentially causing convergence issues.


\subsection{Visual Analysis}

We present a visual analysis of how the proposed regularization term, $\mathcal{L}_{\mathrm{icv}}$, reduces the intra-class variance of graph-level representations. Figure~\ref{fig:intra_curve} illustrates the class-wise intra-class variance during training on PROTEINS and IMDB-BINARY. During the warm-start stage, the variance remains relatively high and changes only slightly, as optimization is primarily driven by the standard cross-entropy loss. However, once the training transitions to the second stage at epoch 100, the variance sharply decreases for both classes on both datasets and continues to decline in subsequent epochs. The emergence of this distinct inflection suggests that $\mathcal{L}_{\mathrm{icv}}$ encourages graph-level representations to concentrate within tighter class clusters, thereby reducing the intra-class variance.

\begin{figure}[ht]
    \centering
    \includegraphics[width=\linewidth, height=4.5cm]{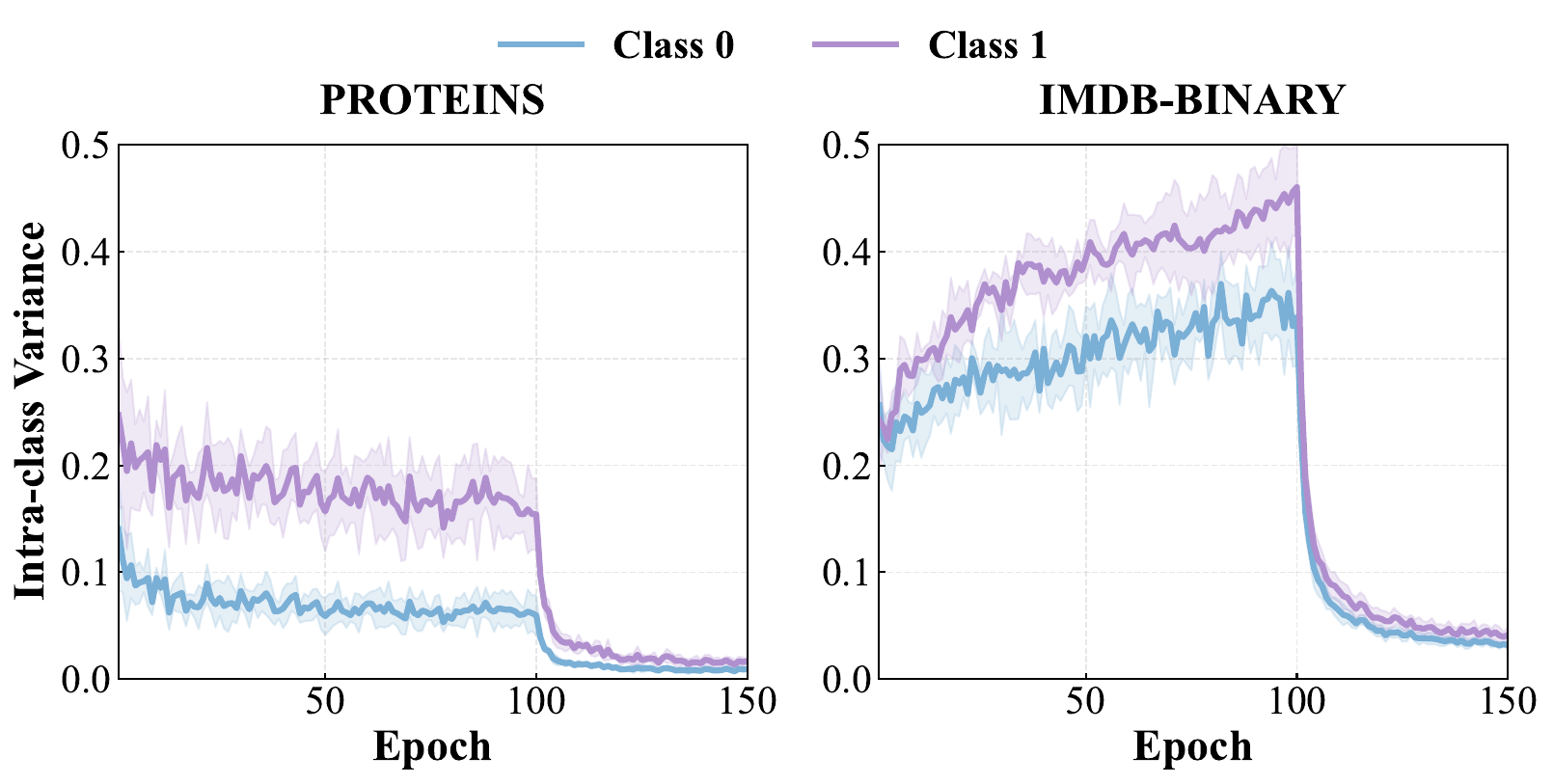}
    \caption{Visual analysis of intra-class variance curves.}
    \label{fig:intra_curve}
\end{figure}

\subsection{Runtime Comparison}
\label{sec:Runtime-Comparison}
We empirically compare the runtime of our model with that of competing baselines in terms of per-epoch training time. Experiments are conducted on PROTEINS and IMDB-BINARY under the same training setting. Figure~\ref{fig:AbstainGNN_training_time} reports the average training time of a single epoch. The results show that AbstainGNN incurs only a moderate computational overhead compared to all baselines, while remaining within an acceptable runtime range. This indicates that the proposed two-stage training strategy introduces limited additional cost, achieving a favorable balance between model effectiveness and computational efficiency.

\begin{figure}[ht]
    \centering
    \includegraphics[width=\linewidth,height=4.5cm]{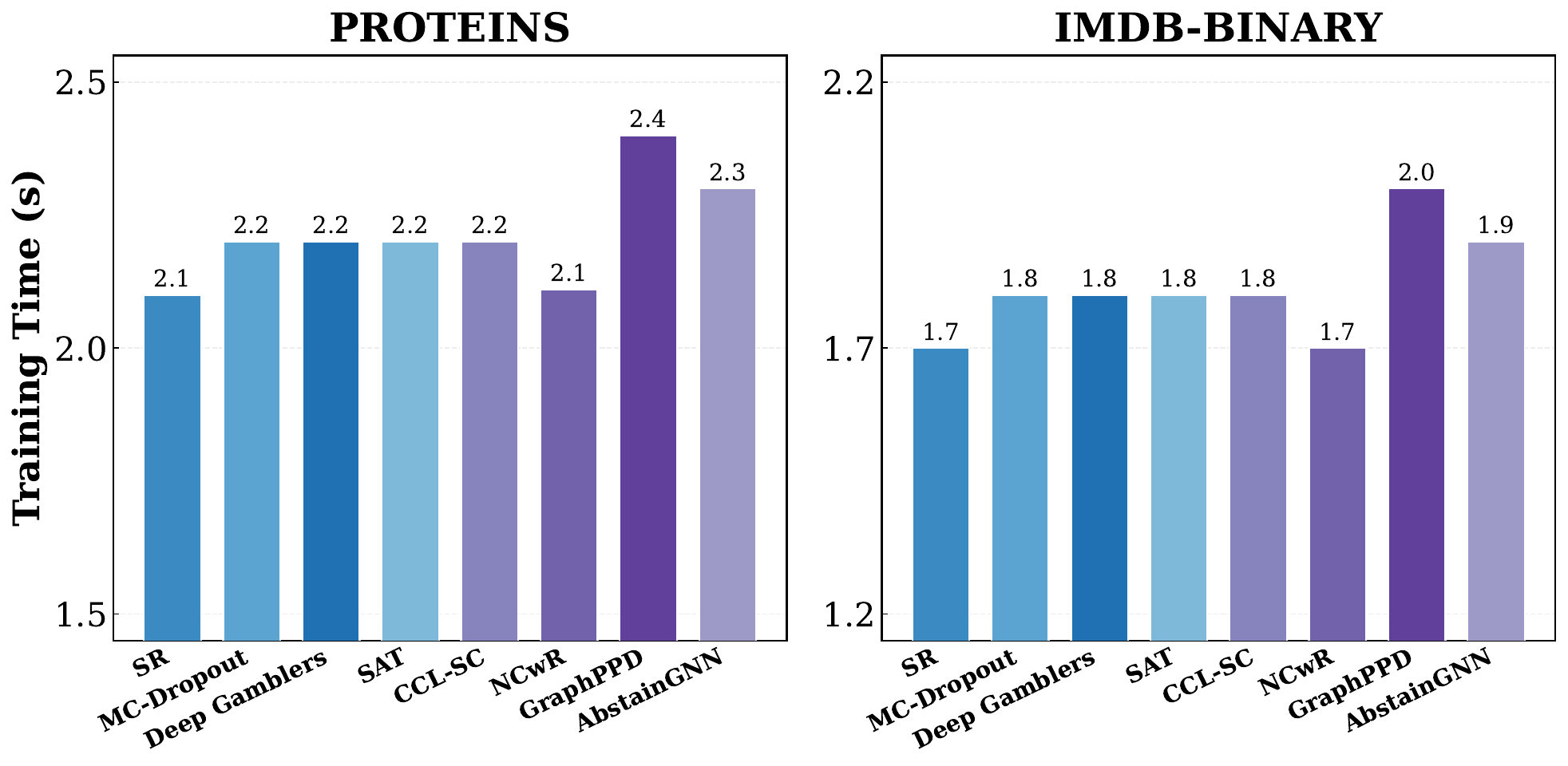}
    \caption{Per-epoch training time comparison.}
    \label{fig:AbstainGNN_training_time}
\end{figure}

\section{Conclusion}
In this paper, we propose AbstainGNN which explicitly equips GNNs with the ability to abstain from making predictions when confidence is insufficient. We provide a rigorous theoretical analysis of AbstainGNN from the perspectives of generalization bounds and training convergence. Guided by these theoretical insights, we further develop an efficient implementation of the proposed framework, including the predictive function warm-start to stabilize loss optimization and the abstention function calibration to ensure well-behaved abstention decisions. We conduct experiments on five benchmark datasets of graph classification. Extensive empirical results and analyses validate both the effectiveness of model designs and the soundness of theoretical derivations. In the future, we plan to extend the proposed abstention mechanism to a wider range of graph-related tasks and explore more suitable model architectures for AbstainGNN.


\begin{acks}
This work is supported by the National Natural Science Foundation of China (No.62402491, No.U2336202, No.62472416, No.62376064, U23A20305, and 62302345), and is also supported by Shanghai Pujiang Program (No.24PJA004). We thank the reviewers for their thoughtful feedback and constructive suggestions, which have helped us substantially improve the manuscript.
\end{acks}

\newpage

\bibliographystyle{ACM-Reference-Format}
\balance
\bibliography{sample-base}

\newpage
\appendix

\section{Appendix}



\subsection{Proof of Theorem \ref{thm:abstention_gen}}
\label{appendix:proof theorem 3.1}
We first establish an intermediate lemma that converts the abstention objective into into an empirical margin-based loss with a complexity term measured. 
\begin{lemma}
\label{lem:pacbayes_selective}
Let \(f_w:\mathcal{G}\to [0,1]^{|\mathcal{Y}|}\) be any predictive function parameterized by \(w\) in Section~\ref{sec:Graph_Classification_with_Abstention} and \(P\) be any prior distribution over the model parameters that is independent of the training sample \(\mathcal{S}\).
Fix \(\rho,\rho',\alpha,\beta,\lambda>0\) and \(\delta\in(0,1)\).
Then, with probability at least \(1-\delta\) over the draw of a training dataset \(\mathcal{S}\) consisting of \(n\) i.i.d. graph samples from the data distribution \(\mathcal{D}\), the following
inequality holds for any \(w\) and any random perturbation \(u\) simultaneously: 
\begin{equation}
\mathbb{P}_{u}\!\left[
\max_{(G,y)\in \mathcal{S}}
\bigl\|f_{w+u}(\mathbf{A}, \mathbf{X})-f_w(\mathbf{A}, \mathbf{X})\bigr\|_2
<
\min\!\left\{\frac{\rho}{4\alpha},\ \frac{\rho'}{4\beta\lambda+2\alpha}\right\}
\right]\ge \frac12.
\end{equation}
Under this condition, we have
\begin{equation}
\begin{aligned}
\mathbb{E}_{(G,y)\sim \mathcal{D}}\!\left[L(f,g;G,y)\right]
\le
\mathbb{E}_{(G,y)\sim \mathcal{S}}\!\left[L_{MH}^{\rho,\rho'}(f,g;G,y)\right]\\
+
4\sqrt{\frac{\mathrm{KL}(w+u\,\|\,P)+\ln\!\left(\frac{6n}{\delta}\right)}{n-1}}\,,
\end{aligned}    
\end{equation}
where \(\mathrm{KL}(w+u\,\|\,P)\) denotes the KL divergence.
\end{lemma}

\begin{proof}
Let $w' = w+u$ and $~\tilde\rho=\min\!\left\{\frac{\rho}{4\alpha},\frac{\rho'}{4\beta\lambda+2\alpha}\right\}$, we define 
\begin{equation}
S_w \;=\;\left\{w' \;:\;
\max_{(G,y)\in\mathcal S}
\bigl\|f_{w'}(\mathbf A,\mathbf X)-f_w(\mathbf A,\mathbf X)\bigr\|_2
<
\tilde\rho
\right\},
\end{equation}
By assumption, $Z \;=\;\mathbb P_u[w+u\in S_w]\;\ge\;\frac12.$
Let $\widetilde Q$ denote the conditional distribution of $w' \in S_w$, and we denote the weight sample $\widetilde w \sim \widetilde Q$ and $\widetilde w \in S_w$ almost surely. Fix any graph–label pair $(G,y)$, and let $p_w(G)=f_w(\mathbf A,\mathbf X)$.
For any $i,j\in\mathcal{Y}$, we have 
\begin{equation}
\begin{aligned}
&\bigl|\,(p_{\widetilde w}(G)[i]-p_{\widetilde w}(G)[j])-(p_w(G)[i]-p_w(G)[j])\,\bigr|\\
&\le |p_{\widetilde w}(G)[i]-p_w(G)[i]|+|p_{\widetilde w}(G)[j]-p_w(G)[j]|\\
&\le 2\|p_{\widetilde w}(G)-p_w(G)\|_2
<2\tilde\rho.
\end{aligned}
\end{equation}
Let the margin be $\gamma_w(G,y)=p_w(G)[y]-\max_{j\neq y}p_w(G)[j].$
Then the above bound implies $|\gamma_{\widetilde w}(G,y)-\gamma_w(G,y)| \;<\;2\tilde\rho \;\le\;\frac{\rho}{2\alpha}.$

 We denote the graph-level abstention score $g_w(\mathbf A, \mathbf X)=g_w(G)$ associated with the model parameters $w$, and assume it is $1$-Lipschitz with respect to $p_w(G)$ under $\|\cdot\|_2$, then 
\begin{equation}
|g_{\widetilde w}(G)-g_w(G)|\;\le\;\|p_{\widetilde w}(G)-p_w(G)\|_2 \;<\;\tilde\rho
\;\le\;\frac{\rho'}{4\beta\lambda+2\alpha}.
\end{equation}
We define the ``halved-margin'' loss $L_{MH}^{\rho/2,\rho'/2}$ by replacing
$(\rho,\rho')$ in $L_{MH}^{\rho,\rho'}$ with $(\rho/2,\rho'/2)$ as follows, 
\begin{equation}
\begin{aligned}
\mathbb E_{(G,y)\sim\mathcal D}\!\left[
L_{MH}^{\rho/2,\rho'/2}(f,g;G,y)\right]=\mathbb{P}_\mathcal D[p(G)\leq\frac{\rho}{2\alpha}]\cdot\\
\mathbb{P}_\mathcal D[g(G)>-\frac{\rho'}{4\beta\lambda+2\alpha}]+\lambda\mathbb{P}_\mathcal D[g(G)>\frac{\rho'}{4\beta\lambda+2\alpha}].
\end{aligned}
\end{equation}
Using the margin controls and the monotonicity of the hinge terms in
$L_{MH}^{\rho,\rho'}$, we obtain the following comparisons: 
\begin{equation}
\begin{aligned}
&\mathbb E_{(G,y)\sim\mathcal D}\!\left[L(f_w,g_w;G,y)\right]
\le
\mathbb E_{(G,y)\sim\mathcal D}\!\left[
L_{MH}^{\rho/2,\rho'/2}(f_{\widetilde w},g_{\widetilde w};G,y)\right],\\
&\mathbb E_{(G,y)\sim\mathcal S}\!\left[
L_{MH}^{\rho/2,\rho'/2}(f_{\widetilde w},g_{\widetilde w};G,y)\right]
\le
\mathbb E_{(G,y)\sim\mathcal S}\!\left[L_{MH}^{\rho,\rho'}(f_w,g_w;G,y)\right].
\end{aligned}
\end{equation}
Applying the PAC-Bayesian bound~\cite{neyshabur2017pac} to the
loss $L_{MH}^{\rho/2,\rho'/2}$, the prior distribution $P$, and the conditional distribution $\widetilde Q$, with the probability at least $1-\delta$ over $\mathcal S$, we have 
\begin{equation}
\begin{aligned}
\mathbb E_{(G,y)\sim\mathcal D}\!\left[
L_{MH}^{\rho/2,\rho'/2}(f_{\widetilde w},g_{\widetilde w};G,y)\right]&\le
\mathbb E_{(G,y)\sim\mathcal S}\!\left[
L_{MH}^{\rho/2,\rho'/2}(f_{\widetilde w},g_{\widetilde w};G,y)\right]\\
&+2\sqrt{\frac{2\,\mathrm{KL}(\widetilde Q\,\|\,P)+\ln\!\left(\frac{2n}{\delta}\right)}{n-1}}.
\end{aligned}
\end{equation}
Since $\widetilde Q$ is the distribution of $w'$ conditioned on $\{w'\in S_w\}$,
a standard condition inequality gives
\begin{equation}
\mathrm{KL}(\widetilde Q\,\|\,P)\;\le\;\mathrm{KL}(w+u\,\|\,P)+\ln\!\left(\frac1Z\right)
\;\le\;\mathrm{KL}(w+u\,\|\,P)+\ln 2.
\end{equation}
The second inequality holds because $Z\ge 1/2$. Absorbing constants into the logarithmic term and using $\sqrt{2(a+b)}\le \sqrt{2a}+\sqrt{2b}$
followed by a coarse simplification, we yield
\begin{equation}
2\sqrt{\frac{2\,\mathrm{KL}(\widetilde Q\,\|\,P)+\ln\!\left(\frac{2n}{\delta}\right)}{n-1}}
\;\le\;
4\sqrt{\frac{\mathrm{KL}(w+u\,\|\,P)+\ln\!\left(\frac{6n}{\delta}\right)}{n-1}}.
\end{equation}
Then, we have 
\begin{equation}
\begin{aligned}
\mathbb E_{(G,y)\sim\mathcal D}\!\left[L(f_w,g_w;G,y)\right]
&\le
\mathbb E_{(G,y)\sim\mathcal S}\!\left[L_{MH}^{\rho,\rho'}(f_w,g_w;G,y)\right]\\
&+
4\sqrt{\frac{\mathrm{KL}(w+u\,\|\,P)+\ln\!\left(\frac{6n}{\delta}\right)}{n-1}}.
\end{aligned}
\end{equation}
\end{proof}

\begin{proof}[Proof of theorem \ref{thm:abstention_gen}]
With Lemma~\ref{lem:pacbayes_selective}, we focus on bounding the KL divergence term \(\mathrm{KL}(w+u\,\|\,P)\) to derive Theorem~\ref{thm:abstention_gen}. 
Let \(C_w^{\epsilon}\in\mathbb{R}^{d\times |\mathcal{Y}|}\) be a random matrix with independent entries that have a mean of 0 and a variance of \(\sigma^2\). We can derive a perturbed predictive function \(f^{\epsilon}\) by replacing the weight parameters \(C_w\) of the classification layer \(C\) with \(C_w+C_w^{\epsilon}\). We write \(u=\max\{d,|\mathcal{Y}|\}\) for convenience. Through a standard tail bound for Gaussian random matrices~\cite{tropp2012user}, for any \(t\ge 0\), we have 
\begin{equation}
\mathbb{P}\!\left(\|C_w^{\epsilon}\|_2 \ge t\right)
\le (d+|\mathcal{Y}|)\exp\!\left(-\frac{t^2}{2\sigma^2\,u}\right)
\le 2u\exp\!\left(-\frac{t^2}{2\sigma^2 u}\right).
\end{equation}
Choosing \(t=\sigma\sqrt{2u\ln(4u)}\) yields $\mathbb{P}\!\left(\|C_w^{\epsilon}\|_2 \le \sigma\sqrt{2u\ln(4u)}\right)\ge \frac12.$ In the remainder, we work on this setting which holds with probability at least \(1/2\). Let \(\mathbf{h}_G = (R\circ U)(\mathbf{A},\mathbf{X})\in\mathbb{R}^d\) be the graph-level representation. Then, we have 
\begin{equation}
\begin{aligned}
f^{\epsilon}(\mathbf{A},\mathbf{X})-f(\mathbf{A},\mathbf{X})
&=(C_w+C_w^{\epsilon})\mathbf{h}_G-C_w\mathbf{h}_G
= C_w^{\epsilon}\mathbf{h}_G.\\
\mathbb{E}\bigl\|f^{\epsilon}(\mathbf{A},\mathbf{X})-f(\mathbf{A},\mathbf{X})\bigr\|_2^2
&=\mathbb{E}\|C_w^{\epsilon}\mathbf{h}_G\|_2^2
\le \mathbb{E}\|C_w^{\epsilon}\|_2^2 \cdot \mathbb{E}\|\mathbf{h}_G\|_2^2.
\end{aligned}
\end{equation}
Moreover, $
\mathbb{E}\|\mathbf{h}_G\|_2^2
=\mathrm{tr}\!\big(\mathbb{E}[\mathbf{h}_G\mathbf{h}_G^\top]\big)
=\mathrm{tr}\!\big(\mathrm{Cov}[\mathbf{h}_G]+\mathbb{E}[\mathbf{h}_G]\mathbb{E}[\mathbf{h}_G]^\top\big).$
Assuming the representations are centered, i.e., \(\mathbb{E}[\mathbf{h}_G]=0\), we obtain $\mathbb{E}\|\mathbf{h}_G\|_2^2=\mathrm{tr}\!\big(\mathrm{Cov}[\mathbf{h}_G]\big).$
By the law of total covariance,
\begin{equation}
\mathrm{Cov}[\mathbf{h}_G]
=\mathbb{E}\!\big[\mathrm{Cov}[\mathbf{h}_G\mid y]\big]
+\mathrm{Cov}\!\big(\mathbb{E}[\mathbf{h}_G\mid y]\big).
\end{equation}
Under the additional standing assumption that classes occur with equal probability \(1/|\mathcal{Y}|\), we have
\begin{equation}
\mathrm{tr}\!\left(\mathbb{E}\!\big[\mathrm{Cov}[\mathbf{h}_G\mid y]\big]\right)
=
\mathrm{tr}\!\left(\frac{1}{|\mathcal{Y}|}\sum_{y\in\mathcal{Y}} \mathrm{Cov}[\mathbf{h}_G\mid y]\right).
\end{equation}
Furthermore, let $\bm{\mu}_y=\mathbb{E}[\mathbf{h}_G\mid y]$ and $I_y$ is a subset of $\mathcal{S}$ that graph samples belonging to the class $y$, the above term can be written as
\begin{equation}
\begin{aligned}
 &\mathrm{tr}\!\left(\frac{1}{|\mathcal{Y}|}\sum_{y\in\mathcal{Y}} \mathrm{Cov}[\mathbf{h}_G\mid y]\right)=\frac{1}{|\mathcal{Y}|}\sum_{y\in \mathcal{Y}}\frac{1}{|I_y|}\sum_{G\in I_y}\bigl\|\mathbf{h}_G-\bm{\mu}_y\bigr\|_2^2,\\
&\mathrm{tr}\!\left(\mathrm{Cov}\!\big(\mathbb{E}[\mathbf{h}_G\mid y]\big)\right)
=
\mathrm{tr}\!\left(\frac{1}{2|\mathcal{Y}|^2}\sum_{y\neq y'}
(\bm{\mu}_y-\bm{\mu}_{y'})(\bm{\mu}_y-\bm{\mu}_{y'})^\top\right).
\end{aligned}
\end{equation}
Using the margin condition encoded by \(\tilde\rho\) and the head norm \(\|C\|_2\), we upper bound this term by
\begin{equation}
\mathrm{tr}\!\left(\mathrm{Cov}\!\big(\mathbb{E}[\mathbf{h}_G\mid y]\big)\right)
\le \frac{(|\mathcal{Y}|-1)\tilde{\rho}^2}{|\mathcal{Y}|\|C\|_2^2}.
\end{equation}
Combining these above and using the spectral-norm event \(\|C_w^{\epsilon}\|_2\le \sigma\sqrt{2u\ln(4u)}\) (it holds with probability \(\ge 1/2\)) and let $\mathrm{Var}_{\mathrm{icv}}(\mathcal{S},\mathcal{Y})=\frac{1}{|\mathcal{Y}|}\sum_{y\in \mathcal{Y}}\frac{1}{|I_y|}\sum_{G\in I_y}\bigl\|\mathbf{h}_G-\bm{\mu}_y\bigr\|_2^2$, we obtain
\begin{equation}
\mathbb{E}\bigl\|f^{\epsilon}(\mathbf{A},\mathbf{X})-f(\mathbf{A},\mathbf{X})\bigr\|_2^2
\le
\sigma^2\,2u\ln(4u)\left(
\mathrm{Var}_{\mathrm{icv}}(\mathcal{S},\mathcal{Y})
+\frac{\tilde{\rho}^2}{\|C\|_2^2}
\right).
\end{equation}

To meet the stability requirement in Lemma~\ref{lem:pacbayes_selective}, we can choose the value of \(\sigma\), enabling the right-hand side to equal \(\tilde\rho^2\): 
\begin{equation}
\sigma^2\,2u\ln(4u)\left(
\mathrm{Var}_{\mathrm{icv}}(\mathcal{S},\mathcal{Y})
+\frac{\tilde{\rho}^2}{\|C\|_2^2}
\right)=\tilde\rho^2,
\end{equation}
which yields $\sigma^2
=\frac{\tilde{\rho}^2}{2u\ln(4u)\left(\mathrm{Var}_{\mathrm{icv}}(\mathcal{S},\mathcal{Y})+\frac{\tilde{\rho}^2}{\|C\|_2^2}\right)}.$ We take a Gaussian prior \(P\) whose scale parameter matches the perturbation variance \(\sigma^2\).
With the above choice of \(\sigma\), the corresponding KL divergence term is bounded by
\begin{equation}
\mathrm{KL}(w+u\,\|\,P)
\le
\frac{u\ln(4u)\left(\mathrm{Var}_{\mathrm{icv}}(\mathcal{S},\mathcal{Y})+\frac{\tilde{\rho}^2}{\|C\|_2^2}\right)}{\tilde{\rho}^2}.
\label{KL}
\end{equation}
Substituting Eq.\eqref{KL} into Lemma~\ref{lem:pacbayes_selective}, we finish the proof.
\end{proof}

\subsection{Proof of Theorem \ref{thm:intra_contraction}}
\label{appendix:proof theorem 3.3}
\begin{proof}
Fix a class \(y\) with index set \(I_y\) and \(|I_y|\ge 1\). For each sample \(G\in I_y\), we denote the cross-entropy gradients by
\begin{equation}
\mathbf{g}_G =\nabla_{\mathbf{h}_G}\mathcal{L}_\mathrm{ce},
\qquad
\bar{\mathbf{g}}_y = \frac{1}{|I_y|}\sum_{G\in I_y} \mathbf{g}_G.
\end{equation}
We define $S_y=\sum_{G\in I_y} w_G\|\mathbf{h}_G-\bm{\mu}_y\|_2^2$, so \(\mathcal{L}_{\mathrm{icv}}=\sum_y \frac{1}{|I_y|}S_y\). By computing
\(\nabla_{\mathbf{h}_G} S_y\) for a fixed class \(y\) and the graph sample \(G\in I_y\), we obtain
\begin{equation}
\begin{aligned}
&\nabla_{\mathbf{h}_G}\mathcal{L}_{\mathrm{icv}}
=\frac{2}{|I_y|}
\Bigl[w_G(\mathbf{h}_G-\bm{\mu}_y)-s_y\Bigr],\\
&s_y =\frac{1}{|I_y|}\sum_{G'\in I_y} w_{G'}(\mathbf{h}_{G'}-\bm{\mu}_y)
=\frac{1}{|I_y|}\sum_{G'\in I_y} w_{G'} \mathbf{d}_{G'} .
\end{aligned}
\end{equation}
With the learning rate \(\eta>0\), the gradient update on \(\mathcal{L}_{\mathrm{total}}\) yields
\begin{equation}
\label{eq:update_h}
\mathbf{h}_G^{+}
= \mathbf{h}_G-\eta \nabla_{\mathbf{h}_G},\quad\mathcal{L}_{\mathrm{total}}
= \mathbf{h}_G-\eta \mathbf{g}_G-\frac{2\lambda_{r}\eta}{|I_y|}\bigl(w_G \mathbf{d}_G-s_y\bigr),
\end{equation}
Averaging Eq.\eqref{eq:update_h} over \(G\in I_y\) and using \(\sum_{G\in I_y} \mathbf{d}_G=0\), we have
\begin{equation}
\label{eq:update_mu}
\bm{\mu}_y^{+}
=\frac{1}{|I_y|}\sum_{G\in I_y} \mathbf{h}_G^{+}
= \bm{\mu}_y-\eta \bar{\mathbf{g}}_y .
\end{equation}
Subtracting Eq.\eqref{eq:update_mu} from Eq.\eqref{eq:update_h} yields the deviation recursion: 
\begin{equation}
\label{eq:update_d}
\mathbf{d}_G^{+}
=\mathbf{h}_G^{+}-\bm{\mu}_y^{+}
=\Bigl(1-\frac{2\lambda_{r}\eta}{|I_y|}w_G\Bigr)\mathbf{d}_G
-\eta\,(\mathbf{g}_G-\bar{\mathbf{g}}_y)
+\frac{2\lambda_{r}\eta}{|I_y|} s_y .
\end{equation}
We define the intra-class variance and introduce some notations:
\begin{equation}
M_y=\frac{1}{|I_y|}\sum_{G\in I_y}\|\mathbf{d}_G\|_2^2,
\quad
M_y^{+} =\frac{1}{|I_y|}\sum_{G\in I_y}\|\mathbf{d}_G^{+}\|_2^2,
\quad
M:=\sum_y M_y,
\end{equation}
\begin{equation}
a_G=1-\frac{2\lambda_{r}\eta}{|I_y|}w_G,
\quad
\mathbf{u}_G = a_G \mathbf{d}_G,
\quad
\mathbf{v}_G =-\eta(\mathbf{g}_G-\bar{\mathbf{g}}_y)+\frac{2\lambda_{r}\eta}{|I_y|} s_y,
\end{equation}
so that Eq.\eqref{eq:update_d} becomes \(\mathbf{d}_G^{+}=\mathbf{u}_G+\mathbf{v}_G\).
The inequality
\(\|\mathbf{u}+\mathbf{v}\|_2^2\le \|\mathbf{u}\|_2^2+\|\mathbf{v}\|_2^2\) gives
\begin{equation}
\label{eq:Vc_plus_basic}
M_y^{+}
\le
\,\frac{1}{|I_y|}\sum_{G\in I_y} a_G^2\|\mathbf{d}_G\|_2^2
+
\,\frac{1}{|I_y|}\sum_{G\in I_y}\|\mathbf{v}_G\|_2^2.
\end{equation}
Since \(w_G\in[\frac{1}{|\mathcal{Y}|},1]\), we have 
\begin{equation}
a_r^2
=\Bigl(1-\frac{2\lambda_{r}\eta}{|I_y|}w_G\Bigr)^2
\le
1-\frac{4\lambda_{r}\eta}{|\mathcal{Y}||I_y|}
+\frac{4\lambda_{r}^2\eta^2}{|I_y|^2}.
\end{equation}
Substituting into the first term of Eq.\eqref{eq:Vc_plus_basic} yields
\begin{equation}
\label{eq:contraction_term}
\,\frac{1}{|I_y|}\sum_{G\in I_y} a_r^2\|\mathbf{d}_G\|_2^2
\le
(1-\frac{4\lambda_{r}\eta}{|\mathcal{Y}||I_y|}
+\frac{4\lambda_{r}^2\eta^2}{|I_y|^2}\bigr)M_y.
\end{equation}
We then show how to bound \(\frac{1}{|I_y|}\sum_{G\in I_y}\|\mathbf{v}_G\|_2^2\).
First, we have
\begin{equation}
\frac{1}{|I_y|}\sum_{G\in I_y}\bigl\|\eta(\mathbf{g}_G-\bar{\mathbf{g}}_y)\bigr\|_2^2
\le 4\eta^2 \max_G||\mathbf{g}_G||^2.
\end{equation}
By Cauchy--Schwarz inequality, we have 
\begin{equation}
\|s_y\|_2
\le
\frac{1}{|I_y|}\sum_{G\in I_y} w_G\|\mathbf{d}_G\|_2
\le
 \frac{1}{|I_y|}\sum_{G\in I_y}\|\mathbf{d}_G\|_2
\le
\sqrt{M_y},
\end{equation}
and therefore $\Bigl\|\frac{2\lambda_{r}\eta}{|I_y|} s_y\Bigr\|_2^2
\le
\frac{4\lambda_{r}^2\eta^2}{|I_y|^2}\,\,M_y.$

Using \(\|a+b\|_2^2\le 2\|a\|_2^2+2\|b\|_2^2\), we obtain
\begin{equation}
\label{eq:disturbance_term}
\frac{1}{|I_y|}\sum_{G\in I_y}\|\mathbf{v}_G\|_2^2
\le
8\eta^2 \max_G||\mathbf{g}_G||^2
+
\frac{8\lambda_{r}^2\eta^2}{|I_y|^2}\,M_y.
\end{equation}
Substituting Eq.\eqref{eq:contraction_term} and Eq.\eqref{eq:disturbance_term} into Eq.\eqref{eq:Vc_plus_basic},
we conclude
\begin{equation}
M_y^{+}
\le (1-\frac{4\lambda_{r}\eta}{|\mathcal{Y}||I_y|}
+\frac{4\lambda_{r}^2\eta^2}{|I_y|^2}\bigr)M_y+8\eta^2 \max_G||\mathbf{g}_G||^2
+
\frac{8\lambda_{r}^2\eta^2}{|I_y|^2}\,M_y
\end{equation}
Summing over all classes, we obtain the global one-step recursion: 
\begin{equation}
\label{eq:global_recursion}
M^{+}\le \bigl(1-\frac{4\lambda_{r}\eta}{|\mathcal{Y}|\min_y|I_y|}
+\frac{12\lambda_{r}^2\eta^2}{\max_y|I_y|^2})M + 8\max_G||\mathbf{g}_G||^2\,\eta^2.
\end{equation}
If $\eta$ is a constant and \(0<\eta<\min\{\frac{\max_y|I_y|^2}{6k\lambda_{r}\min_y|I_y|},\frac{|\mathcal{Y}|\min_y|I_y|}{2\lambda_{r}}\}\), then \(1-\frac{4\lambda_{r}\eta}{|\mathcal{Y}|\min_y|I_y|}
+\frac{12\lambda_{r}^2\eta^2}{\max_y|I_y|^2}\le 1-\frac{2\lambda_{r}}{|\mathcal{Y}|\min_y|I_y|}\eta = \rho_{\eta}\in(0,1)\),
and iterating \eqref{eq:global_recursion} yields
\begin{equation}
M^{(t)}\le \rho_{\eta}^{t}M^{(0)}+\frac{8\eta^2}{1-\rho_{\eta}}\max_{G,t}||\mathbf{g}^{(t)}_G||^2
=\rho_{\eta}^{t}M^{(0)}+\mathcal{O}(\eta)\max_{G,t}||\mathbf{g}^{(t)}_G||^2,
\end{equation}
which proves geometric convergence to an \(\mathcal{O}(\eta)\) neighborhood.
\end{proof}
\end{document}